\documentclass{article} % For LaTeX2e
\usepackage{iclr2025_conference,times}

\usepackage{bbm}
\usepackage{physics}

\usepackage{dutchcal}

\usepackage[utf8]{inputenc}
\usepackage[english]{babel}
\usepackage{polski}
\usepackage{helvet}
\usepackage{graphicx}
\usepackage{color}
\usepackage{geometry}
\usepackage{url}
\usepackage[hidelinks]{hyperref}
\usepackage{minted}
\usepackage{amsmath}
\usepackage[T1]{fontenc}
\usepackage{amssymb}
\usepackage{textcomp}
\usepackage{tabularx}
\usepackage{array}

\usepackage[style=numeric,backend=biber, maxbibnames=10, sorting=none]{biblatex}
\usepackage{enumitem}
\usepackage{epstopdf}

\usepackage{multicol}
\usepackage{caption}
\usepackage{subcaption}
\usepackage{booktabs}
\usepackage{adjustbox}

\usepackage{multirow}
\usepackage{longtable}

\usepackage{amsthm}
\newtheorem{lemma}{Lemma}

\usepackage{IEEEtrantools}

\usepackage{xcolor}

\definecolor{goodblue}{HTML}{4477AA}
\definecolor{badorange}{HTML}{EE7733}
\definecolor{darkgreen}{RGB}{0, 100, 0}

\newcommand{\englanfz}[1]{``#1''}

\title{Quantifying Ambiguity in Categorical Annotations: A Measure and Statistical Inference Framework}
\author{Christopher Klugmann\thanks{Corresponding author} ~and Daniel Kondermann \\
	Quality-Match GmbH\\
	Heidelberg, Germany\\
	\texttt{\{ck,dk\}@quality-match.com} \\
}

\iclrfinalcopy
\begin{document}

\maketitle

\begin{abstract}
	Human-generated categorical annotations frequently produce empirical response distributions (soft labels) that reflect ambiguity rather than simple annotator error. We introduce an ambiguity measure that maps a discrete response distribution to a scalar in the unit interval, designed to quantify aleatoric uncertainty in categorical tasks. The measure bears a close relationship to quadratic entropy / Gini-style impurity but departs from those indices by treating an explicit \englanfz{can't solve} category asymmetrically---thereby separating uncertainty arising from class-level indistinguishability from uncertainty due to explicit unresolvability. We analyze the measure's formal properties and contrast its behaviour with a representative ambiguity measure from the literature. Moving beyond description, we develop statistical tools for inference: we propose frequentist point estimators for population ambiguity, and we derive the Bayesian posterior over ambiguity induced by Dirichlet priors on the underlying probability vector, providing a principled account of epistemic uncertainty. Numerical examples illustrate estimation, calibration, and practical use for dataset-quality assessment and downstream machine-learning workflows. Implementations of the ambiguity measures and the code for reproducing the numerical results are available at \url{https://github.com/cklugmann/ambiguity_paper}.
\end{abstract}

\section{Introduction}

\textsc{Crowdsourcing} has emerged as a valuable tool for generating labeled data for machine learning applications across a wide range of domains.  
Of particular importance within this paradigm are categorical annotation tasks, where annotators assign discrete labels by choosing from a fixed set of predefined options.  
Such tasks are at the core of many classification problems, including widely studied examples in computer vision and natural language processing---such as labeling social media posts for hate speech or toxicity.

While seemingly straightforward, the assumption of a single, unambiguous \englanfz{ground truth} label per instance is often challenged in practice \cite{aroyo2015truth}.  
As noted by Inel et al.\ in their \textsc{CrowdTruth} paper \cite{inel2014crowdtruth}, human interpretation of annotation tasks is frequently subjective and influenced by factors such as data quality. This aligns with findings by Peterson et al.\ \cite{peterson2019human}, who argue for explicitly incorporating ambiguity into predictive modeling by moving away from hard classification and instead training on full label distributions for greater robustness.
Similarly, de Vries and Thierens \cite{devries2024learningconfidencetrainingbetter} demonstrate that training with soft labels can yield consistently better classifiers, showing improved calibration and resilience to noise compared to models trained on hard labels.
More recent work by Prabhakaran et al.\ \cite{prabhakaran2021releasing} and Davani et al. \cite{davani2022dealing} further emphasizes the importance of preserving raw, annotator-level labels in NLP datasets, arguing that aggregation can obscure valuable perspectives and introduce representation biases.

Soft labels---discrete probability distributions over possible labels---are often more representative of the true nature of the data than single hard labels.  
However, they are also more difficult to work with directly.  
It is therefore natural to consider summary statistics that capture key aspects of these distributions, such as the degree of aleatoric uncertainty inherent in the task or object being annotated.

To this end, we propose a novel scalar ambiguity measure designed to capture both label uncertainty and the possibility of task-level unsolvability.  
Our approach builds on ideas from information theory, in particular the quadratic entropy \cite{vajda1968bounds}, also known as \textit{Gini impurity} in decision trees \cite{Breiman1984} or the \textit{Gini--Simpson index} in ecology \cite{jost2006entropy}.  
We extend this concept to the annotation setting, including support for an explicit \englanfz{can't solve} option.  
The resulting measure enables efficient filtering, ranking, and exploration of data based on instance-level ambiguity (see Figure~\ref{fig:discrete-task-example}).

\begin{figure}
    \centering
    \includegraphics[width=0.95\linewidth]{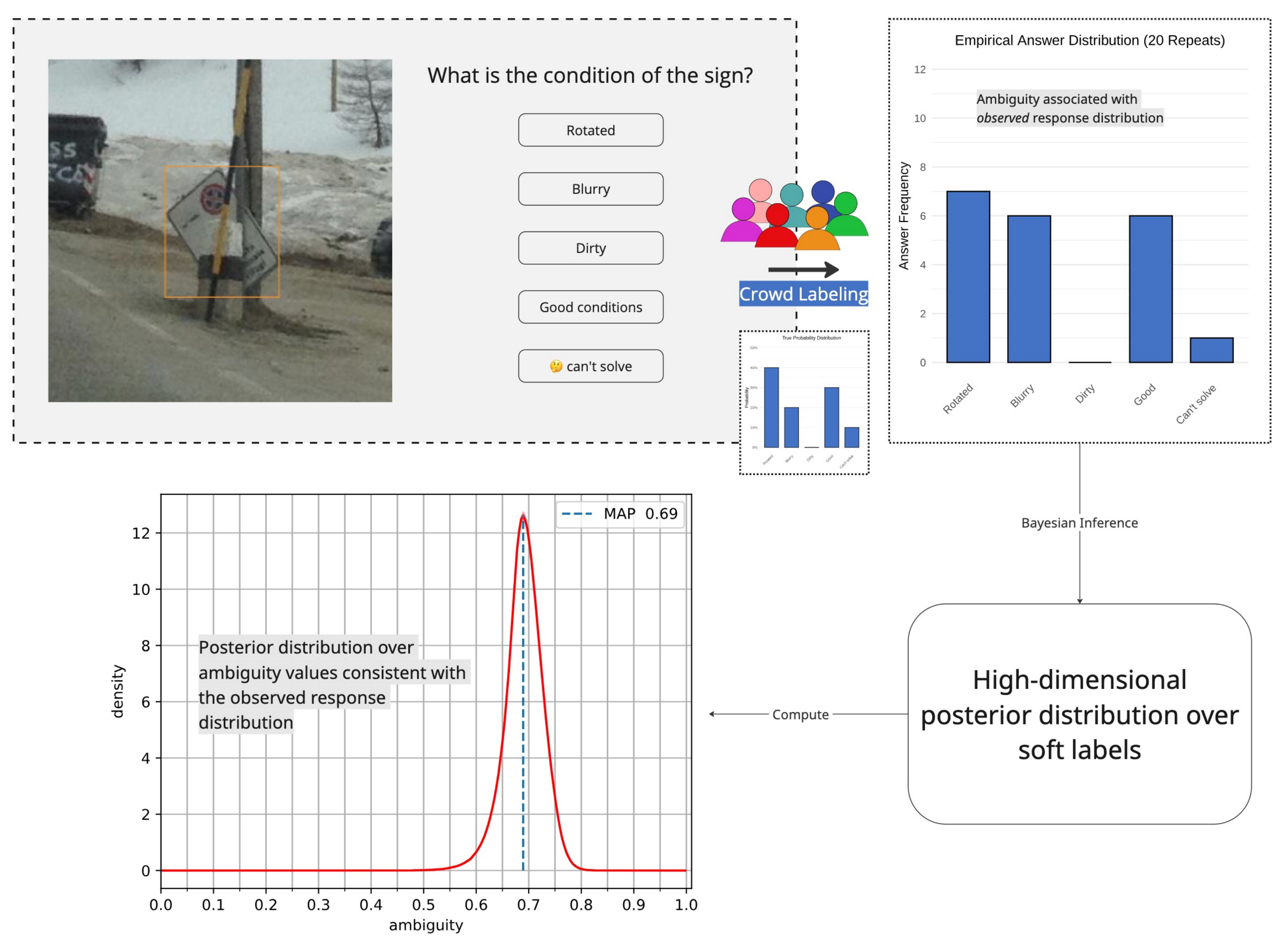}
    \caption{A labeling task is performed by a group of annotators (the \textit{crowd}), resulting in a distribution over possible answers. In this example, the distribution is derived from $20$ individual responses, reflecting high uncertainty regarding a single correct label. This observed response distribution embodies both \textit{epistemic uncertainty}---due to the finite sample of annotations---and \textit{aleatoric uncertainty}, which captures the intrinsic, irreducible uncertainty linked to the task and the annotators. The latter is what we define as \textit{ambiguity}, the central focus of this work. Using Bayesian inference, we estimate the posterior distribution over possible ambiguity values, quantifying how compatible each is with the observed data.}
    \label{fig:discrete-task-example}
\end{figure}

In addition to defining and motivating our ambiguity measure (and a modified variant), we study its mathematical properties and examine how it behaves in practical annotation scenarios.  
Crucially, we place this work in a statistical framework: we investigate how to estimate ambiguity from finite annotation samples using Bayesian inference.  
This allows us to disentangle aleatoric (task-inherent) uncertainty from epistemic (sampling) effects and quantify posterior uncertainty over ambiguity values given observed data.

\section{Definition of the Ambiguity Measure}

Let $\vb{q}$ denote the probability vector representing the possible outcomes of a single categorical annotation task. We aim to focus on the uncertainty or ambiguity encoded by $\vb{q}$. We define ambiguity in natural language as follows:
\begin{quote}
The probability that an annotator deems a task unsolvable, or that even if they consider it solvable, a second annotator arrives at a different answer, assuming the second annotator also considers the task solvable.
\end{quote}
This definition implies that ambiguity arises from two sources: the general solvability or unsolvability of the task on the one hand, and the indistinctness of the remaining categories on the other.
Note also that in our natural language definition, we use the concept of an annotator, who is typically associated with certain characteristics such as the ability to solve a task correctly. However, we abstract away from these traits and treat annotators as operators who work in a statistically indistinguishable manner, providing answers according to the same probability distribution $\mathrm{Cat}(\vb{q})$.

The definition of ambiguity as a probability immediately suggests the event we need to examine. Let $Y_1, Y_2 \sim \mathrm{Cat}(\vb{q})$ be independent random variables. We then define the ambiguity for $\vb{q}$ as the probability of the event
\begin{align*}
\{ Y_1 = \mathrm{cs} \} \cup \{Y_1 \neq \mathrm{cs}, Y_1 \neq Y_2 \} \quad \text{given that} \quad Y_2 \neq \mathrm{cs}\text{.}
\end{align*}
Here, $\mathrm{cs}$ denotes the designated response category indicating task unsolvability (\englanfz{\textbf{c}an't \textbf{s}olve}).
Using the independence of the events separated by $\cup$ and the independence of $Y_1$ and $Y_2$, the probability of the event is written as
\begin{align}
    \mathrm{amb} &\equiv P(Y_1 = \mathrm{cs}) + P( Y_1 \neq \mathrm{cs}, Y_1 \neq Y_2 \mid Y_2 \neq \mathrm{cs}) \nonumber \\
    &= q_{\mathrm{cs}} + (1 - q_{\mathrm{cs}}) P(Y_1 \neq Y_2 \mid Y_1 \neq \mathrm{cs}, Y_2 \neq \mathrm{cs})\text{.} \label{eq:def_amb}
\end{align}
To make the calculation easier to follow, we introduce the conditional probability vector $\vb{p}$ (over the $C$ categories), defined by
\begin{align}
    p_k \equiv P(Y = k \mid Y \neq \mathrm{cs}) = q_k / (1 - q_{\mathrm{cs}}), \quad k = 1, \dots, C\text{,} 
\end{align}
where the probability is computed with respect to $Y \sim \mathrm{Cat}(\vb{q})$. With this, we see that
\begin{align}
    P(Y_1 \neq Y_2 \mid Y_1 \neq \mathrm{cs}, Y_2 \neq \mathrm{cs}) &= \sum_{k=1}^C P(Y_1 = k, Y_2 \neq k \mid Y_1 \neq \mathrm{cs}, Y_2 \neq \mathrm{cs}) \nonumber\\
    &= \sum_{k=1}^C P(Y_1 = k \mid Y_1 \neq \mathrm{cs}) P(Y_2 \neq k \mid Y_2 \neq \mathrm{cs}) \nonumber\\
    &= \sum_{k=1}^C p_k (1-p_k) = 1 - \sum_{k=1}^C p_k^2\text{.} \label{eq:labelflip_prob}
\end{align}
Substituting into \eqref{eq:def_amb} and performing simple algebraic manipulations yields
\begin{align}
    \mathrm{amb} = 1 - \frac{1}{1-q_{\mathrm{cs}}} \sum_{k=1}^C q_k^2
    \quad \text{if} \, q_{\mathrm{cs}} < 1\text{.}
    \label{eq:conditional_amb}
\end{align}
In the case where the entire probability mass is concentrated on the $\texttt{cs}$ category, we set $\mathrm{amb}$ to 1. This makes sense, as we want to assign the maximum ambiguity to the scenario of complete unsolvability. Thus, we finally obtain
\begin{align}
    \mathrm{amb}(\vb{q}) = \mathbbm{1}[q_{\mathrm{cs}} = 1] + \mathbbm{1}[q_{\mathrm{cs}} < 1] \left\{ 1 - \frac{1}{1-q_{\mathrm{cs}}} \sum_{k=1}^C q_k^2 \right\}\text{,}
    \label{eq:def_amb_final}
\end{align}
where $\mathbbm{1}[A(x)]$ is simply the Iverson bracket that is $1$ if the predicate $A$ is true for $x$ and $0$ otherwise.
By convention, an expression of the form $\mathbbm{1}[\texttt{False}]\varphi(x)$ always evaluates to zero, even if $\varphi$ is not defined at the point $x$.

\subsection{Maximum uniform distributions do not have greatest ambiguity}

After constructing the ambiguity measure, equation \eqref{eq:def_amb_final}, it is clear that $\mathrm{amb}$ is in any case in the interval $[0, 1]$, where values closer to $1$ indicate greater ambiguity. It is illustrative to look at extreme cases of the probability vectors to get a feeling for the values that the ambiguity measure actually takes.

In the case where $\vb{q} = ((1-q_{\mathrm{cs}})/C, \dots, (1-q_{\mathrm{cs}})/C, q_{\mathrm{cs}})^{\intercal}$, we obtain $\mathrm{amb} = 1 - (1-q_{\mathrm{cs}})/C$. This means that the \englanfz{severity} of a completely uniform distribution depends on the number of categories from which the annotators can choose. 
Moreover, the ambiguity depends on the degree of unsolvability, regardless of the uniformity among the other categories. Ambiguity is at its lowest when no probability mass is assigned to the $\texttt{cs}$ category, resulting in a value of $(C-1)/C$. Conversely, ambiguity increases as $q_{\mathrm{cs}}$ approaches $1$.

The special case where $q_{\mathrm{cs}} = 0$, meaning that the entire probability mass is distributed among the remaining categories, provides valuable insight. As mentioned, the ambiguity in this case is $(C-1)/C$, which asymptotically approaches $1$.
For the dichotomous case, where $C=2$, we instead have an ambiguity of $1/2$. This may seem implausible, given that the majority vote under such a distribution is maximally uncertain. However, this is consistent with the definition of ambiguity based on label flip probabilities. The probability that a labeler chooses a different answer than the first annotator is the probability that they select the remaining response that the first annotator did not choose, which is $50\%$ under the assumed uniformity among the proper categories.
In other words, if we were to survey not just two, but $n+1$ annotators and examine how often the label changes between trial $i$ and trial $i+1$ for $i=1, \dots, n$, we would expect this to occur approximately $50\%$ of the time.

\subsection{The Modified Ambiguity Measure}

The fact that a probability vector of $(1/2, 1/2, 0)^{\intercal}$ does not exhibit maximum ambiguity in the dichotomous case is understandable in light of the label flip interpretation. However, one could argue that this behavior of the chosen measure is not fully aligned with the general understanding of ambiguity. Therefore, we now propose a modification to the ambiguity measure, which directly follows from the previous considerations.

Semantically, we no longer base the definition of ambiguity on label flip probabilities alone, but rather on \textit{normalized} label flip probabilities. These are probabilities that are forced to take values in the entire interval $[0, 1]$ by being scaled according to the maximum expected rate of label flips for a given number of categories. We observed that the probability in Equation \eqref{eq:labelflip_prob} for a uniform $\vb{p}$ takes the value $(C-1)/C$. Intuitively, it is clear that this must represent the maximum achievable value across all $\vb{p}$. Using the Cauchy-Schwarz inequality, it is also straightforward to explicitly verify that $(C-1)/C$ is indeed an upper bound. It holds that
\begin{align}
    C \sum_{k=1}^C p_k^2 = \norm{\pmb{1}_C}_2^2 \, \norm{\vb{p}}_2^2 \ge \left( \vb{p}^{\intercal} \pmb{1}_C \right)^2 = \left(\sum_{k=1}^C p_k \right)^2 = 1\text{.}
\end{align}
Here, $\pmb{1}_C$ simply denotes the vector of length $C$, which contains only ones. We therefore see that $\sum_{k=1}^C p_k^2 \ge 1/C$, which provides the desired upper bound for \eqref{eq:labelflip_prob}.
The final step is to appropriately adjust the ambiguity measure by normalizing the label flip probability. This gives us the \textit{modified ambiguity measure} $\widetilde{\mathrm{amb}}$ as
\begin{align}
    \widetilde{\mathrm{amb}}(\vb{q}) = \begin{cases}
        q_{\mathrm{cs}} + \frac{C}{C-1} \left\{ (1 - q_{\mathrm{cs}}) - \frac{1}{1-q_{\mathrm{cs}}} \sum_{k=1}^C q_k^2\right\} & \text{if} \, q_{\mathrm{cs}} < 1 \\
        1 & \text{else.}
    \end{cases}
    \label{eq:def_modified_amb}
\end{align}
Note that the modified ambiguity measure can be expressed in terms of the unmodified ambiguity measure as
\begin{align}
    \widetilde{\mathrm{amb}}(\vb{q}) = \frac{1}{C-1} \left( C\cdot\mathrm{amb}(\vb{q}) - q_{\mathrm{cs}} \right)\text{,}
    \label{eq:rel_mod_unmod}
\end{align}
which can be easily verified by substituting \eqref{eq:def_amb_final} into \eqref{eq:def_modified_amb} and performing straightforward algebraic manipulations.

It is easy to see that the modified ambiguity is never smaller than the unmodified ambiguity, i.e., $\widetilde{\mathrm{amb}}(\vb{q}) \ge \mathrm{amb}(\vb{q})$ for all $\vb{q}$.
For the case $q_{\mathrm{cs}} = 1$, this is immediately clear, so without loss of generality, we can restrict our attention to the case where $q_{\mathrm{cs}} < 1$. In that case, $\widetilde{\mathrm{amb}} \ge \mathrm{amb}$ holds if and only if $\mathrm{amb} \ge q_{\mathrm{cs}}$ is satisfied. Now, we have \begin{align} \mathrm{amb} = 1 - \frac{1}{1-q_{\mathrm{cs}}} \sum_{k=1}^C q_k^2 \ge q_{\mathrm{cs}} \iff (1-q_{\mathrm{cs}})^2 \ge \sum_{k=1}^C q_k^2. \end{align}

Since $(1 - q_{\mathrm{cs}})^2 = \left( \sum_{k=1}^C q_k \right)^2$, by the binomial theorem, the inequality holds for every $\vb{q}$, which proves the statement.

\subsection{Qualitative Examples}

We now aim to develop a better understanding of which probability vectors lead to specific values of ambiguity. To do so, we will examine various discrete probability distributions and compare the values produced by the ambiguity measure $\mathrm{amb}$ and the modified ambiguity measure $\widetilde{\mathrm{amb}}$.
In addition, we include in our comparison a related ambiguity measure, denoted as $\mathrm{amb}_0$, which has been employed in prior work by Schwirten et al.\ \cite{luisa_paper} and Klugmann et al.\ \cite{aint_paper}.
For reference, we briefly recall this definition. Using the notation introduced earlier, the ambiguity of a probability vector $\vb{q}$ under this alternative definition can be written as \begin{align}
\mathrm{amb}_0(\vb{q}) = \mathbbm{1}[q_{\mathrm{cs}} = 1] + \mathbbm{1}[q_{\mathrm{cs}} < 1] \left\{ 1 - \frac{1 - q_{\mathrm{cs}}}{2} \frac{C}{C-1} \sum_{k=1}^{C} \left| p_k - \frac{1}{C} \right| \right\}\text{.}
\end{align}
As in the case of the modified ambiguity measure, the normalization constant $(C-1)/C$ appears in the definition of $\mathrm{amb}_0$. When multiplied by a factor of 2, this represents the maximum value that the distance, defined by the sum of component-wise absolute differences between $\vb{p}$ and $(1/C)\pmb{1}_C$, can take.
This sum of absolute distances can be regarded as a special case of a $\phi$-divergence between the (conditional) probability vector $\vb{p}$ and a uniform distribution over $C$ classes, the \textit{total variation} \cite{vajda2007generalized}.

\paragraph{Dichotomous case.}

In Figure \ref{fig:binary_amb} (left), we show eight examples of dichotomous distributions (not counting the $\texttt{cs}$ category), numbered from (1) to (8). To the right, a table presents the values provided by the different definitions of ambiguity for each distribution. \textit{New} refers to the new ambiguity measure $\mathrm{amb}$, \textit{modified} to the label-flip-normalized version $\widetilde{\mathrm{amb}}$, and \textit{old} to the alternative measure $\mathrm{amb}_0$. As theoretically noted, the modified measure never yields smaller values than $\mathrm{amb}$. Furthermore, the numbers suggest that $\mathrm{amb}_0$ generally falls between $\mathrm{amb}$ and $\widetilde{\mathrm{amb}}$, which holds in many cases, though not necessarily in all, as we will demonstrate with additional examples later.

Another notable feature, which has been theoretically proven, is that $\mathrm{amb}$ can reach values up to $0.5$ for distributions that are completely uniform along the proper categories, such as distribution (6). In contrast, the other two ambiguity measures yield significantly larger and more intuitively aligned values. From the perspective of both the modified and old measures, distributions (6) through (8) are equally ambiguous. In the case of distributions (6) and (7), we observe maximum uniformity, while distribution (8) reflects maximum unsolvability, both of which lead to maximum values under $\mathrm{amb}_0$ and $\widetilde{\mathrm{amb}}$. However, $\mathrm{amb}$ only produces maximum ambiguity when the underlying distribution suggests complete unsolvability.

Additionally, we observe that the modified ambiguity measure gives relatively large values, such as in the case of distribution (3), with a value of $0.64$, even when the majority response, with $\max_k q_k = 0.8$, remains fairly identifiable.

\begin{figure}[ht]
    \centering
    \begin{minipage}[t]{0.55\textwidth}\vspace{0pt}%
        %\centering
        \includegraphics[width=\textwidth]{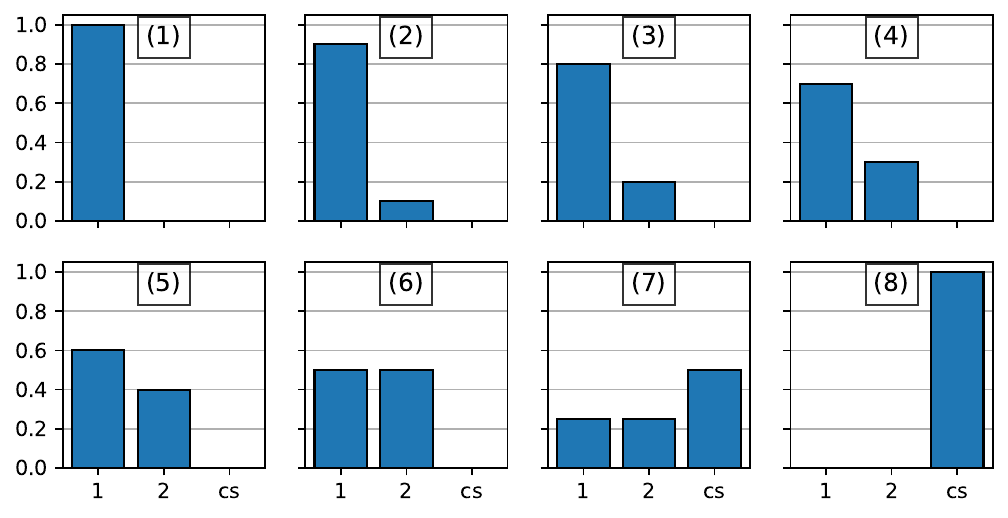}
    \end{minipage}
    \hspace{0.05\textwidth}
    \begin{minipage}[t]{0.2\textwidth}\vspace{0pt}%
        \footnotesize
        \centering
        %\captionof{table}{Comparison of the different ambiguity measures for distributions (1) - (8).}
        \begin{tabular}{cccc}
            \toprule
            & \textbf{new} & \textbf{old} & \textbf{modified} \\
            \midrule
            1 & 0.00 & 0.00 & 0.00 \\
            2 & 0.18 & 0.20 & 0.36 \\
            3 & 0.32 & 0.40 & 0.64 \\
            4 & 0.42 & 0.60 & 0.84 \\
            5 & 0.48 & 0.80 & 0.96 \\
            6 & 0.50 & 1.00 & 1.00 \\
            7 & 0.75 & 1.00 & 1.00 \\
            8 & 1.00 & 1.00 & 1.00 \\
            \bottomrule
        \end{tabular}
    \end{minipage}
    \caption{Eight examples of dichotomous distributions (including the \texttt{cs} category) with varying degrees of ambiguity (left). The table (right) shows the ambiguity values of distributions (1) - (8) for the new ($\mathrm{amb}$), old ($\mathrm{amb}_0$), and modified ($\widetilde{\mathrm{amb}}$) ambiguity measures.}
    \label{fig:binary_amb}
\end{figure}

\paragraph{General categorical case.}

\begin{figure}[ht]
    \centering
    \begin{minipage}[t]{0.55\textwidth}\vspace{0pt}%
        %\centering
        \includegraphics[width=\textwidth]{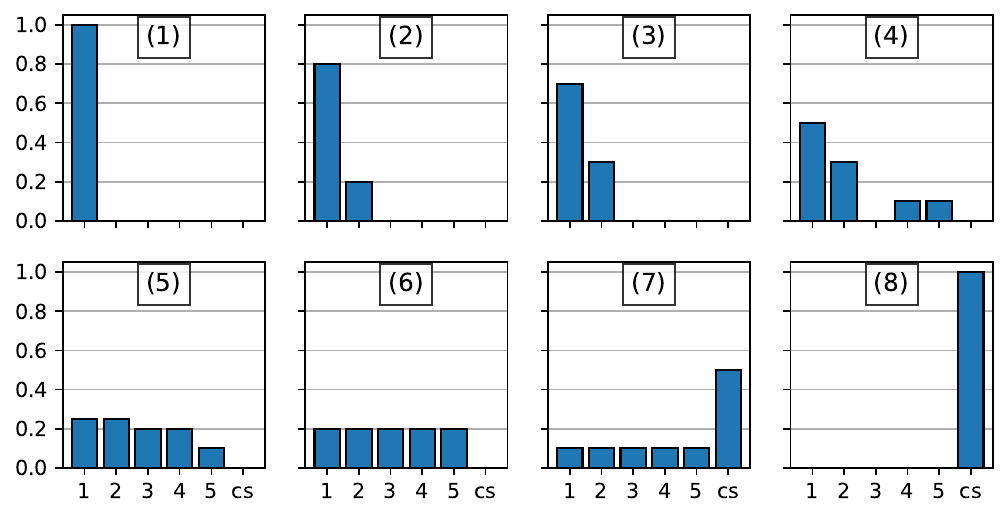}
    \end{minipage}
    \hspace{0.05\textwidth}
    \begin{minipage}[t]{0.2\textwidth}\vspace{0pt}%
        \footnotesize
        \centering
        \begin{tabular}{cccc}
            \toprule
            & \textbf{new} & \textbf{old} & \textbf{modified} \\
            \midrule
            1   & 0.000 & 0.000 & 0.000 \\
2   & 0.320 & 0.250 & 0.400 \\
3   & 0.420 & 0.250 & 0.525 \\
4   & 0.640 & 0.500 & 0.800 \\
5   & 0.785 & 0.875 & 0.981 \\
6   & 0.800 & 1.000 & 1.000 \\
7   & 0.900 & 1.000 & 1.000 \\
8   & 1.000 & 1.000 & 1.000 \\

            \bottomrule
        \end{tabular}
    \end{minipage}
    \caption{Eight examples of categorical distributions (including the \texttt{cs} category) with varying degrees of ambiguity (left). The table (right) shows the ambiguity values of distributions (1) - (8) for the new ($\mathrm{amb}$), old ($\mathrm{amb}_0$), and modified ($\widetilde{\mathrm{amb}}$) ambiguity measures.}
    \label{fig:categorical_amb}
\end{figure}

As before, we examine eight distributions, but now they are general categorical distributions with $C=5$ categories (excluding the $\texttt{cs}$ category). These distributions are shown in Figure \ref{fig:categorical_amb}, and the corresponding ambiguity values under the three different measures are presented in the table on the right. From the table, we immediately observe that, as hinted earlier, the old ambiguity measure does not necessarily yield values that lie between those of the other two measures. 

We also note that the old measure assigns identical values of $0.25$ to distributions (2) and (3), even though we would intuitively attribute greater ambiguity to distribution (3) compared to (2). The new and the (new) modified measures behave more as expected in this regard, though we again observe that the modified measure quickly takes on larger values. Conversely, distribution (6) demonstrates that $\mathrm{amb}$ assigns a value less than 1 to a fully uniform distribution (without any $\texttt{cs}$ component), specifically $0.8$. As anticipated, the old and modified measures coincide for distributions (6) through (8) due to the previously discussed behavior, indicating maximum ambiguity.

\paragraph{Summary of comparison.}
We summarize the results of our comparison in Table \ref{tab:comparison}, where the three ambiguity measures are contrasted regarding their identified advantages and disadvantages.

\begin{table}[h]
    \centering
    \footnotesize
    \begin{tabular}{p{5cm} p{5cm} p{5cm}} % Definierte Spaltenbreiten
        \toprule
        \textbf{New}, $\mathrm{amb}$ & \textbf{Old}, $\mathrm{amb}_0$ & \textbf{Modified}, $\widetilde{\mathrm{amb}}$ \\
        \midrule
        \begin{itemize}
            \item[+] \textcolor{goodblue}{Interpretation}
            \item[-] \textcolor{badorange}{Not maximal for uniform distributions}
        \end{itemize} & 
        \begin{itemize}
            \item[+] \textcolor{goodblue}{Moderate, often plausible values}
            \item[+] \textcolor{goodblue}{Maximal for uniform distributions}
            \item[-] \textcolor{badorange}{Some distributions are considered equivalently ambiguous, even though they shouldn’t be}
            \item[-] \textcolor{badorange}{No straightforward interpretation as probability or similar}
        \end{itemize} & 
        \begin{itemize}
            \item[+] \textcolor{goodblue}{Still a strong interpretation}
            \item[+] \textcolor{goodblue}{Maximal for uniform distributions}
            \item[-] \textcolor{badorange}{Provides too large ambiguity values too quickly}
        \end{itemize} \\
        \bottomrule
    \end{tabular}
    \caption{Summary---comparison of ambiguity measures.}
    \label{tab:comparison}
\end{table}

\section{Inference}

Up to this point, we have assumed that $\vb{q}$ is a given probability vector, for which we only need to deterministically compute ambiguity. However, in real-world applications, we do not know $\vb{q}$. Instead, all we have are the responses $Y_1, \dots, Y_R$ observed from our annotators. There are two main schools of thought in statistics to proceed from this point.

On one hand, frequentist statistics assumes that there is a fixed but unknown parameter vector $\vb{q}$, on which the realizations of the observed responses are based. We can estimate $\vb{q}$ by providing a point estimate, simply calculating the empirical relative frequencies with which each category was selected. This point estimator, $\hat{\vb{q}}$, is subject to the sampling distribution of $(Y_1, \dots, Y_R)$, meaning that different realizations can occur under repeated (hypothetical) experiments. Using standard techniques, one can derive confidence intervals instead of a single point estimate, which cover the true probability vector $\vb{q}$ with a certain probability. However, the downside of this approach is its analytical complexity and the difficulty most users face when interpreting confidence intervals or confidence regions.

The other major school of thought, Bayesian statistics, offers a remedy. Here, $\vb{q}$ itself is treated as a random variable, and a joint probability model $p(\vb{q}, \vb{y})$ is considered, where $\vb{y}$ denotes the observations. Instead of seeking a point estimate for $\vb{q}$, the goal is to assess the plausibility of every possible value of this random variable, given the observed $\vb{y}$. Mathematically, this means determining the posterior distribution $p(\vb{q} \, \vert \, \vb{y})$.

Fortunately, there is a simple joint probability model in which the conditional probability of $\vb{q}$ can be derived analytically. In this model, $\vb{q}$ is treated as a random variable distributed according to $\mathrm{Dir}(\pmb{\alpha}_0)$, a Dirichlet distribution. The distribution of $\vb{y}$ given $\vb{q}$ is then just the product of $\mathrm{Cat}(y_i; \vb{q})$ terms. It turns out that the posterior distribution of $\vb{q}$ given $\vb{y}$ also follows a Dirichlet distribution: $p(\vb{q} \, \vert \, \vb{y}) = \mathrm{Dir}(\pmb{\alpha}_0 + \vb{n})$, where $n_k = \sum_{r=1}^R \mathbbm{1}[y_r = k]$.

This also makes it clear how to determine the distribution of any derived random variable $f(\vb{q})$, where $f$ is any function, in a posterior sense. All that is required is to repeatedly sample $\vb{q}_i \sim \mathrm{Dir}(\pmb{\alpha})$, where $\pmb{\alpha} = \pmb{\alpha}_0 + \vb{n}$, and then apply the function $f$, yielding a sample $f_i = f(\vb{q}_i)$ for $i = 1, \dots, n$. Based on this sample, one can approximate the posterior mean, mode, or credible intervals for $f(\vb{q})$.
In our case, the function $f$ generally corresponds to one of the ambiguity measures for which we aim to determine uncertainty intervals. Before we tackle this task, we first aim to develop better intuition about the properties of random variables realized under $\mathrm{Dir}(\pmb{\alpha})$.

\subsection{Dirichlet Distributions and Entropy}

How does one study which probability vectors $\vb{p}$\footnote{Note that here we use $\vb{p}$ instead of $\vb{q}$. We do this because, in this paragraph, we do not specifically refer to probability vectors arising from discrete annotation tasks. Also, $\vb{p}$ does not correspond to the conditional probability discussed earlier. However, the conditional probability vector resembles a general probability vector (without an emphasized probability for unsolvability), which is why we use this notation.} appear most plausible under a Dirichlet distribution $\mathrm{Dir}(\pmb{\alpha})$? One way is to examine the density function. For distributions over three categories, ternary plots are commonly used, which display the density of the Dirichlet distribution over the standard simplex, as shown in Figure \ref{fig:dir_heatmap}. For higher-dimensional problems, such visualizations are less accessible. Instead, we could examine the marginal distributions of the individual $p_k$ or reduce the probability vectors $\vb{p}$ by applying a scalar function. Since we want to study ambiguity further, the latter is a natural choice.

A tool deeply rooted in information theory to gain insight into any probability distribution is the \textit{information entropy} of a random variable, henceforth referred to simply as entropy. Like ambiguity, entropy measures the uncertainty inherent in a random variable. For a categorical random variable (or equivalently, for the underlying probability vector $\vb{p}$ of the distribution), entropy is defined as $\mathrm{H} = \sum_k p_k \log(1/p_k)$. It is immediately apparent that entropy is non-negative in all cases (since $p_k$ are probabilities and thus not greater than $1$). Denoting the number of categories by $C$, we can also determine a maximum entropy, which is given by $\log C$. For our purposes, it is useful to consider a normalized version of entropy, which we obtain as $\widetilde{\mathrm{H}} = \mathrm{H} / \log C$, i.e., by dividing the entropy by its maximum possible value.

By reducing a probability vector $\vb{p}$ to its (normalized) entropy $\widetilde{\mathrm{H}}(\vb{p})$, we hope to obtain a property of $\vb{p}$ that is easier to study under the assumed distribution of $\vb{p}$. The normalized entropy is high, i.e., close to $1$, when the probability vector $\vb{p}$ exhibits high uniformity. For degenerate distributions, where $\vb{p}$ is \englanfz{one-hot}, the entropy is $0$. For Dirichlet distributions with concentrated parameters in one category, we expect the realized probability vectors $\vb{p}$ to assign a large probability mass to the distinguished category with high probability, resulting in low entropy. For parameter vectors that are not concentrated, we expect, with greater probability, more uniform probability vectors, which exhibit high entropy.

To illustrate how this intuition can sometimes be misleading, even in the case of $C=3$ categories, consider the example of $\vb{p} \sim \mathrm{Dir}(11, 2, 2)$, whose distribution is shown in Figure \ref{fig:dir_heatmap}. Note that this corresponds to the case where we observe responses $\vb{n} = (10, 1, 1)^{\intercal}$ and wish to infer the plausibility of $\vb{p}$ based on a uniform prior $\pmb{\alpha}_0 = (1, 1, 1)^{\intercal}$. The absolute majority of responses---here, $10$---fall into the first category, while the remaining two categories account for only one response each. The sample-based approximation of the posterior distribution of $\widetilde{\mathrm{H}}(\vb{p})$ is shown in Figure \ref{fig:entropy_dist}. Surprisingly, the mode and expected value of the normalized entropy lie above $0.6$. The distribution has a large spread, but it is slightly left-skewed, meaning it has a heavier tail towards smaller values.

In fact, the expected value of $\widetilde{\mathrm{H}}(\vb{p})$ under $\vb{p} \sim \mathrm{Dir}(\pmb{\alpha})$ can even be computed analytically. It is given by
\begin{align}
E\left(\widetilde{\mathrm{H}}(\vb{p})\right) = \frac{1}{\log C} \left\{ \psi(\alpha_0 + 1) - \sum_{k=1}^C \frac{\alpha_k}{\alpha_0} \psi(\alpha_k + 1)\right\},
\end{align}
where $\alpha_0 = \sum_{k=1}^C \alpha_k$ and $\psi$ denotes the digamma function. In this case, the expected value is approximately $0.64$, which matches the observed sample mean to three decimal places.
% The normalized entropy of the mode of $\mathrm{Dir}(\pmb{\alpha})$, i.e., $\widetilde{\mathrm{H}}(\vb{m})$ (with $\vb{m} = (\pmb{\alpha} - \pmb{1}_C) / (\alpha_0 - C)$), in this specific case is approximately $0.52$.

\begin{figure}[h]
    \centering
    \begin{subfigure}[t]{0.47\textwidth}
        \centering
        \includegraphics[width=\textwidth]{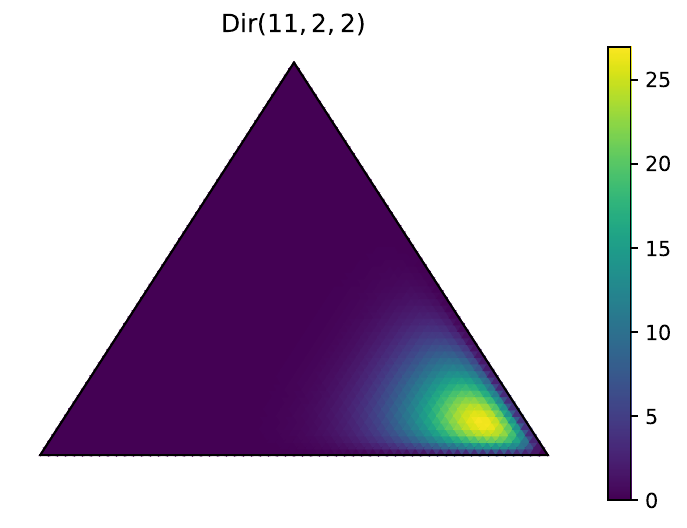}
        \caption{The density of a $\mathrm{Dir}(11, 2, 2)$ distributed random variable, represented as a heatmap over the standard simplex $\Delta_2$}
        \label{fig:dir_heatmap}
    \end{subfigure}
    \hspace{0.01\textwidth}
    \begin{subfigure}[t]{0.47\textwidth}
        \centering
        \includegraphics[width=\textwidth]{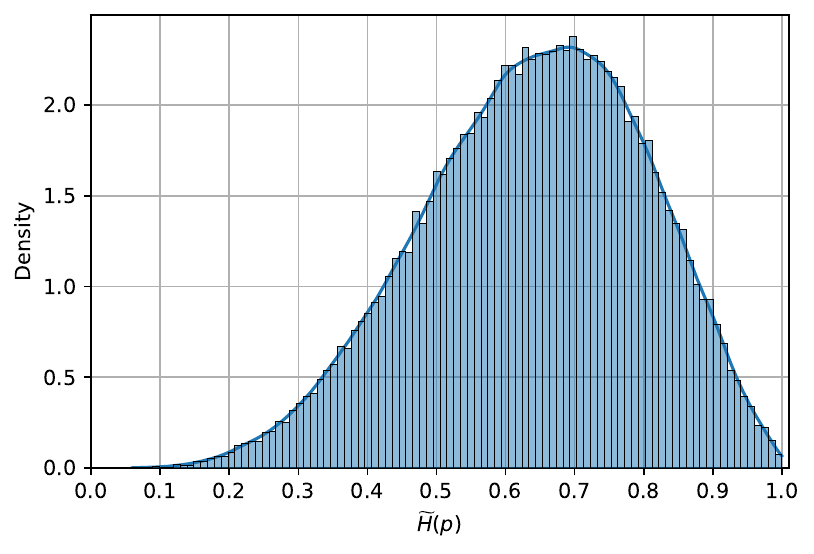}
        \caption{Approximate distribution of the normalized entropy $\widetilde{\mathrm{H}}(\vb{p})$ for $\vb{p} \sim \mathrm{Dir}(11, 2, 2)$.}
        \label{fig:entropy_dist}
    \end{subfigure}
    \caption{Example of a Dirichlet distribution over $C=3$ categories, shown in the left half of the image as a heatmap over the standard simplex. The right half of the image displays the corresponding univariate distribution of the normalized entropy for this distribution.}
\end{figure}

\subsection{Analytic Properties of Ambiguity}

We aim to examine the posterior distribution of ambiguity values while remaining within the conjugate multinomial-Dirichlet model. Specifically, we consider $\vb{q} \sim \mathrm{Dir}(\pmb{\alpha})$ and seek to study how the derived random variable $f(\vb{q})$ is distributed, where $f$ represents one of the proposed ambiguity measures.
Due to the appearance of complicated integrals over subsets of the standard simplex $\Delta_C$ in the general categorical setting, it is not straightforward to obtain a closed-form expression for the distribution of $f(\vb{q})$. However, in the special case of binary classification tasks---i.e., when $C+1 = 3$ including the residual \texttt{cs} category---the integral can indeed be reduced to a univariate Beta integral. The exact form of the resulting probability density functions is derived and presented in Appendix \ref{closed_form_binary_posterior}.

In the general case, we must rely on sampling-based methods to approximate the posterior distribution of $f(\vb{q})$. However, before fully turning to such numerical techniques, we first establish closed-form expressions for the first and second moments, as well as the variance of $f(\vb{q})$, valid for arbitrary values of $C$.
We begin by determining an analytical expression for the expected value $E\left(\mathrm{amb}(\vb{q})\right)$, in analogy to the expected entropy. Under a non-degenerate Dirichlet distribution, we have $P(q_{\mathrm{cs}} = 1) = 0$, which implies that the expectation can be written as
\begin{align} E(\mathrm{amb}(\vb{q})) = E(q_{\mathrm{cs}}) + E\left( (1 - q_{\mathrm{cs}}) \left\{ 1 - \sum_{k=1}^C p_k^2 \right\} \right)\text{.}
\label{eq:exp_amb}
\end{align}
A standard result states that the conditional probability vector $\vb{p}$ and the unsolvability probability $q_{\mathrm{cs}}$ are independent. Moreover, the marginal distributions of these random variables are straightforward to determine. Specifically, we have
\begin{align}
q_{\mathrm{cs}} \sim \mathrm{Beta}(\alpha_{\mathrm{cs}}, \alpha_0 - \alpha_{\mathrm{cs}}), \quad
\vb{p} \sim \mathrm{Dir}(\alpha_1, \dots, \alpha_k)\text{,}
\end{align}
where $\alpha_0 = \alpha_{\mathrm{cs}} + \sum_{k=1}^C \alpha_k$.
Using the independence of the random variables, their marginal distributions, and the analytical expressions for the first and second moments of a Beta-distributed random variable, see appendix \ref{beta_dist} and equations \eqref{eq:first_and_second_moment}, it follows from equation \eqref{eq:exp_amb} that
\begin{align}
E\left(\mathrm{amb}(\vb{q})\right) &= \frac{\alpha_{\mathrm{cs}}}{\alpha_0} + \frac{\alpha_0 - \alpha_{\mathrm{cs}}}{\alpha_0} \left\{ 1- \sum_{k=1}^C E\left(p_k^2\right) \right\} \nonumber \\
%&= 1 - \frac{\alpha_0 - \alpha_{\mathrm{cs}}}{\alpha_0} \sum_{k=1}^C \frac{\alpha_k}{ (\alpha_0 - \alpha_{\mathrm{cs}})^2 } \left\{ \alpha_k + \frac{\alpha_0 - \alpha_{\mathrm{cs}} - \alpha_k }{\alpha_0 - \alpha_{\mathrm{cs}} + 1} \right\} \nonumber \\
&= 1 - \frac{\alpha_0 - \alpha_{\mathrm{cs}}}{\alpha_0} \sum_{k=1}^C \frac{\alpha_k \, (\alpha_k + 1)}{(\alpha_0 - \alpha_{\mathrm{cs}})(\alpha_0 - \alpha_{\mathrm{cs}}+1)} \nonumber \\
&= 1 - \left[ \alpha_0 (\alpha_0 - \alpha_{\mathrm{cs}} + 1) \right]^{-1} \sum_{k=1}^C \alpha_k (\alpha_k + 1)\text{.}\label{eq:expr_amb_final}
\end{align}
As before, the analytical expression for the expected value allows for a comparison with the sampling-based approach. The expected value of the modified ambiguity measure, $E(\widetilde{\mathrm{amb}}(\vb{q}))$, can be directly obtained from equation \eqref{eq:expr_amb_final}, while taking into account the relationship given by \eqref{eq:rel_mod_unmod}.
On the other hand, the expected value under the old ambiguity measure $\mathrm{amb}_0$ cannot be expressed in an elementary form and requires the use of the incomplete Beta function due to the appearance of absolute values.

Beyond the mean, we are certainly also interested in a measure of the dispersion for the respective posterior distribution of ambiguity. A normal distribution would be fully characterized by its first two moments. However, for the distributions arising from the ambiguity measures, the assumption of normality does not generally hold. Nevertheless, the second moment of the distribution, and in particular the second central moment (variance), remains informative to gain a sense of the uncertainty in the ambiguity estimate. For both $\mathrm{amb}$ and $\widetilde{\mathrm{amb}}$, we can write an analytical expression for the variance.

To derive this, we first consider the (non-central) second moment of $\mathrm{amb}(\vb{q})$. It is given by
\begin{equation}
\begin{split}
E\left( \mathrm{amb}^2 \right) = E\left( q_{\mathrm{cs}}^2 \right) &+ 2 E\left( q_{\mathrm{cs}} \, (1 - q_{\mathrm{cs}}) \right)  E\left( 1 - \sum_{k=1}^C p_k^2 \right) \\
&+ E\left( \left[1-q_{\mathrm{cs}}\right]^2 \right) E\left( \left[ 1 - \sum_{k=1}^C p_k^2\right]^2 \right),
\label{eq:second_moment_first}
\end{split}
\end{equation}
where we have once again used the independence of $q_{\mathrm{cs}}$ and $\vb{p}$.
The expressions for the expectations involving $q_{\mathrm{cs}}$ are immediately available from Equations \eqref{eq:first_and_second_moment} and \eqref{eq:mixed_prod_expectation} in Appendix \ref{beta_dist}. The expectation $E\left(\sum_{k=1}^C p_k^2\right)$ was previously calculated when determining $E(\mathrm{amb})$. It can be expressed as
\begin{align}
E\left(\sum_{k=1}^C p_k^2\right) = \frac{\alpha_0}{\alpha_0 - \alpha_{\mathrm{cs}}} \left[1 - E\left(\mathrm{amb}\right)\right]\text{.}
\label{eq:exp_of_sum_of_squares}
\end{align}
This further yields
\begin{align}
E\left( \left[ 1 - \sum_{k=1}^C p_k^2\right]^2 \right) =
1 - 2 \frac{\alpha_0}{\alpha_0 - \alpha_{\mathrm{cs}}} \left[1 - E\left(\mathrm{amb}\right)\right]
+ \sum_{k=1}^C \sum_{l=1}^C E\left( p_k^2 p_l^2\right)\text{.}
\label{eq:squared_term}
\end{align}
For the last term in \eqref{eq:squared_term}, we can use formula \eqref{eq:prod_moments_dirichlet} from Appendix \ref{dirichlet_dist}. This gives us,
\begin{equation}
\begin{split}
E\left( p_k^2 p_l^2 \right) = &\left[
  (\alpha_0 - \alpha_{\mathrm{cs}})
  (\alpha_0 - \alpha_{\mathrm{cs}} + 1)
  (\alpha_0 - \alpha_{\mathrm{cs}} + 2)
  (\alpha_0 - \alpha_{\mathrm{cs}} + 3)
\right]^{-1} \\
&\times \begin{cases}
    \alpha_k(\alpha_k+1)(\alpha_k+2)(\alpha_k+3) & \text{if}\, k=l\text{,} \\
    \alpha_k(\alpha_k+1)\alpha_l(\alpha_l+1) & \text{else.}
\end{cases}
\label{eq:prod_squared_dirichlet}
\end{split}
\end{equation}
Thus, we have all the components needed to determine the expectation of $\mathrm{amb}^2$. After some algebraic transformations, we see that it can be expressed as
\begin{align}
E\left(\mathrm{amb}^2\right) = \mathcal{R} + \mathcal{S} \cdot \left[ 1 - E\left(\mathrm{amb}\right) \right]^2 + 2 \cdot E\left(\mathrm{amb}\right) - 1\text{,}
\label{eq:exp_amb_squared}
\end{align}
where we introduce the abbreviations $\mathcal{R}$ and $\mathcal{S}$, which represent the expressions
\begin{align}
\mathcal{R} &\equiv \frac{
    \sum_{k=1}^C \alpha_k (\alpha_k + 1) \left[ (\alpha_k + 2)(\alpha_k + 3) - \alpha_k(\alpha_k + 1)\right]
}{
    \alpha_0 (\alpha_0 + 1) (\alpha_0 - \alpha_{\mathrm{cs}} + 2)(\alpha_0 - \alpha_{\mathrm{cs}} + 3) 
} \quad \text{and} \\
\mathcal{S} &\equiv \frac{
\alpha_0(\alpha_0 - \alpha_{\mathrm{cs}} + 1)^2
}{
(\alpha_0+1)(\alpha_0 - \alpha_{\mathrm{cs}} + 2)(\alpha_0 - \alpha_{\mathrm{cs}} + 3)\text{.}
}
\end{align}
Using \eqref{eq:exp_amb_squared}, we see that the variance of $\mathrm{amb}$ can be written as
\begin{align}
\mathrm{Var}\left(\mathrm{amb}\right) = E\left( \mathrm{amb}^2 \right) - \left( E\left(\mathrm{amb}\right) \right)^2 = \mathcal{R} + \left( \mathcal{S} - 1\right) \cdot \left[1 - E\left( \mathrm{amb}\right)\right]^2\text{.}
\label{eq:expr_amb_var_final}
\end{align}

By knowing the variance of $\mathrm{amb}$, we can now straightforwardly derive the variance of the modified ambiguity measure $\widetilde{\mathrm{amb}}$. Unsurprisingly, we start with Equation \eqref{eq:rel_mod_unmod}, which relates $\widetilde{\mathrm{amb}}$ to $\mathrm{amb}$. Using the properties of $\mathrm{Var}$, we obtain
\begin{align}
\mathrm{Var}\left( \widetilde{\mathrm{amb}} \right) = \left( \frac{1}{C-1} \right)^2 \left[ C^2 \cdot \mathrm{Var}(\mathrm{amb}) + \mathrm{Var}\left( q_{\mathrm{cs}} \right) - 2 \cdot C \cdot \mathrm{Cov}\left(\mathrm{amb}, q_{\mathrm{cs}}\right)\right]\text{.}
\label{eq:var_amb_tilde}
\end{align}
The covariance term appears because $\mathrm{amb}$ and $q_{\mathrm{cs}}$ are not uncorrelated. Given the independence of $q_{\mathrm{cs}}$ and $\vb{p}$, we can simplify the covariance as follows:
\begin{align}
\mathrm{Cov}\left(\mathrm{amb}, q_{\mathrm{cs}}\right)
&= \mathrm{Var}(q_{\mathrm{cs}}) + E\left( 1 - \sum_{k=1}^C p_k^2 \right) \mathrm{Cov}(1-q_{\mathrm{cs}}, q_{\mathrm{cs}}) \nonumber \\
&= E\left(\sum_{k=1}^C p_k^2\right) \mathrm{Var}(q_{\mathrm{cs}})
= \frac{\alpha_{\mathrm{cs}}}{\alpha_0 (\alpha_0 + 1)}\left[1 - E\left(\mathrm{amb}\right)\right]\text{,}
\label{eq:expr_amb_mod_var}
\end{align}
where we used $\mathrm{Cov}(1-q_{\mathrm{cs}}, q_{\mathrm{cs}}) = -\mathrm{Var}(q_{\mathrm{cs}})$ and basic properties of $\mathrm{Cov}$. The expectation $E\left(\sum_{k=1}^C p_k^2\right)$ was determined in Equation \eqref{eq:exp_of_sum_of_squares}, and the variance of $q_{\mathrm{cs}}$---see equation \eqref{eq:var_beta} in Appendix \ref{beta_dist}---is
\begin{align}
\mathrm{Var}(q_{\mathrm{cs}}) = \frac{\alpha_{\mathrm{cs}} (\alpha_0 - \alpha_{\mathrm{cs}})}{\alpha_0^2(\alpha_0+1)}\text{.}
\end{align}
Substituting these into equation \eqref{eq:var_amb_tilde} provides a closed-form expression for the variance of $\widetilde{\mathrm{amb}}$.

\subsection{A Remark on Frequentist Estimation}

As mentioned initially, there is an alternative approach to the Bayesian school of thought in statistics, the frequentist perspective. Frequentist tools are often used without consciously recognizing their nature. One such example is to determine a point estimate $\hat{\vb{q}}$ and to substitute it into the formula for calculating the ambiguity. The result is a point estimator for the transformed parameter $\mathrm{amb}(\vb{q})$, known as the \textit{plug-in estimator}. For an observation $\vb{n} \sim \mathrm{Multinomial}(n, \vb{q})$ and the point estimator $\hat{\vb{q}}_n = \vb{n} / n$, the plug-in estimator for $\mathrm{amb}(\vb{q})$ is given by
\begin{align} \mathrm{amb}(\hat{\vb{q}}_n) = 1 - \frac{\mathbbm{1}[n_{\mathrm{cs}} < n]}{n(n-n_{\mathrm{cs}})} \sum_{k=1}^C n_k^2\text{.} \end{align}
It can be shown that this estimator is not unbiased, meaning that, in general, $E\left(\mathrm{amb}(\hat{\vb{q}})\right) \neq \mathrm{amb}(\vb{q})$. This is not a specific property of our chosen ambiguity function. In fact, it can be broadly shown that plug-in estimators for \textit{power entropies}---such as the quadratic entropy appearing in the definition of the ambiguity measure (also known as the Gini--Simpson index or Gini impurity)---are generally biased \cite{liitiainen2009statistical}.

The expected value of the plug-in estimator for the ambiguity can be derived analytically.
In the case where the true underlying distribution is degenerate and $q_{\mathrm{cs}} = 1$, the ambiguity, according to our definition, must equal $1$. In this scenario, only degenerate observations can occur, where all responses fall into the $\mathrm{cs}$ category, and the plug-in estimator likewise yields the constant value $1$.
Therefore, we only need to consider the non-trivial case where $q_{\mathrm{cs}} < 1$. In this case, the true ambiguity is given by Equation \eqref{eq:conditional_amb}. For the expected value of the estimator, the following holds:
\begin{align}
    E\left( \mathrm{amb}(\hat{\vb{q}}_n)\right) = 1 - \frac{1}{n} \sum_{l=0}^{n-1} E\left( \frac{\mathbbm{1}[n_{\mathrm{cs}} = l]}{n-n_{\mathrm{cs}}} \sum_{k=1}^C n_k^2\right)\text{,}
    \label{eq:freq_exp_interm}
\end{align}
where we have used that $\mathbbm{1}[n_{\mathrm{cs}} < n] = \sum_{l=0}^{n-1} \mathbbm{1}[n_{\mathrm{cs}} = l]$. The expected value can be rewritten as
\begin{align}
    E\left( \frac{\mathbbm{1}[n_{\mathrm{cs}} = l]}{n-n_{\mathrm{cs}}} \sum_{k=1}^C n_k^2\right) &= E\left( \left. \frac{1}{n-n_{\mathrm{cs}}} \sum_{k=1}^C n_k^2 \, \right\vert \, n_{\mathrm{cs}} = l\right) \varrho_l \nonumber \\
    &= \frac{\varrho_l}{n-n_{\mathrm{cs}}} \sum_{k=1}^C E\left( n_{k}^2 \mid n_{\mathrm{cs}} = l\right)\label{eq:intermed_cond_exp}
\end{align}
where $\varrho_l = P\left( n_{\mathrm{cs}} = l\right)$.
Using the fact that
\begin{align*}
E(n_k^2 \mid n_{\mathrm{cs}} = l) = \mathrm{Var}(n_k \mid n_{\mathrm{cs}} = l) + E(n_k \mid n_{\mathrm{cs}} = l)^2\text{,}
\end{align*}
for all $k = 1, \dots, C$, and that
\begin{align*}
(n_1, \dots, n_C) \mid n_{\mathrm{cs}} = l \sim \mathrm{Multinomial}\left( n - l, \left( \frac{q_1}{1-q_{\mathrm{cs}}}, \dots, \frac{q_C}{1-q_{\mathrm{cs}}}\right) \right)\text{,}
\end{align*}
we obtain from Equation \eqref{eq:intermed_cond_exp}:
\begin{align}
    E\left( \frac{\mathbbm{1}[n_{\mathrm{cs}} = l]}{n-n_{\mathrm{cs}}} \sum_{k=1}^C n_k^2\right) = \frac{(n-l)\, q_k}{(1-q_k)^2}\left[ (1-q_{\mathrm{cs}}-q_k) + (n-l) \, q_k \right]\text{.}
\end{align}
By substitution and simple manipulations, the expectation in \eqref{eq:freq_exp_interm} can ultimately be expressed as
\begin{align}
    E\left( \mathrm{amb}(\hat{\vb{q}}_n)\right) = \left(1 - \frac{1-q_{\mathrm{cs}}^n}{n}\right) - \left[ \frac{1}{1-q_{\mathrm{cs}}}  - \frac{1-q_{\mathrm{cs}}^n}{n(1-q_{\mathrm{cs}})^2} \right] \sum_{k=1}^C q_k^2\text{.}
    \label{eq:plugin_estimator_mean}
\end{align}
As is immediately apparent, the plug-in estimator is not unbiased for the true ambiguity. However, the sequence $( \hat{\vb{q}}_n )_n$ is \textit{consistent}, meaning that for $n \to \infty$, we have $E(\mathrm{amb}(\hat{\vb{q}}_n)) \to \mathrm{amb}(\vb{q})$.

\begin{figure}
    \centering
    \includegraphics[width=0.98\linewidth]{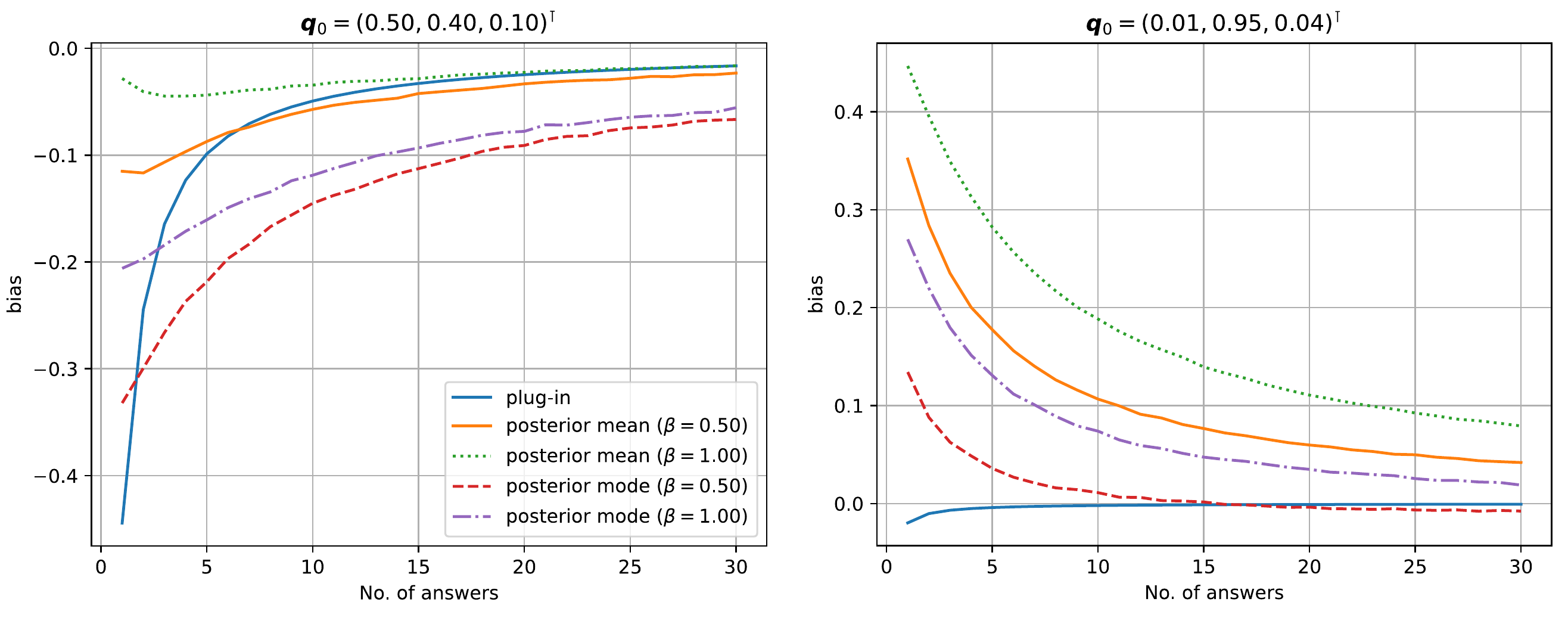}
    \caption{Bias of different estimators of ambiguity as a function of the number of observed answers. Shown are the plug-in estimator and Bayesian estimators (posterior mean and mode) with symmetric Dirichlet priors parameterized by $\beta \in \{0.5, 1.0\}$. Each panel corresponds to a different underlying categorical distribution $\vb{q}_0$, as indicated in the titles. The left panel shows a relatively balanced distribution over two proper categories and one residual \texttt{cant\_solve} category. In contrast, the right panel reflects a highly skewed distribution with most mass concentrated on a single category. In all cases, the estimators are biased but consistent: the bias decreases with increasing sample size. For small numbers of answers, Bayesian estimators are visibly influenced by the prior. Notably, the plug-in estimator underestimates the true ambiguity in expectation and converges from below.}
    \label{fig:biases}
\end{figure}

Moreover, it is easy to see that the plug-in estimator does not \textit{overestimate} the true ambiguity (in expectation). Except for degenerate cases, the sequence
$
\mathrm{bias}_n = E\left(\mathrm{amb}(\hat{\vb{q}}_n)\right) - \mathrm{amb}(\vb{q})
$
is strictly monotonically increasing from below and converges to zero. The proof is provided in Appendix~\ref{plugin_consistency}.

In contrast, estimators based on the Bayesian framework---such as the posterior mean or mode---can also be employed. For various values of $\beta$, which parameterizes the uninformative prior $\mathrm{Dir}(\beta, \dots, \beta)$, the bias of specific Bayesian estimators is shown in Figure~\ref{fig:biases}, alongside the bias of the plug-in estimator for two illustrative examples. These Bayesian estimators are likewise biased and consistent, but their convergence behavior differs: in the left panel of Figure~\ref{fig:biases}, they approach the true ambiguity from below, whereas in the right panel, they do so from above.
\subsection{Choice of Prior}

\begin{figure}
    \centering
    \includegraphics[width=0.98\linewidth]{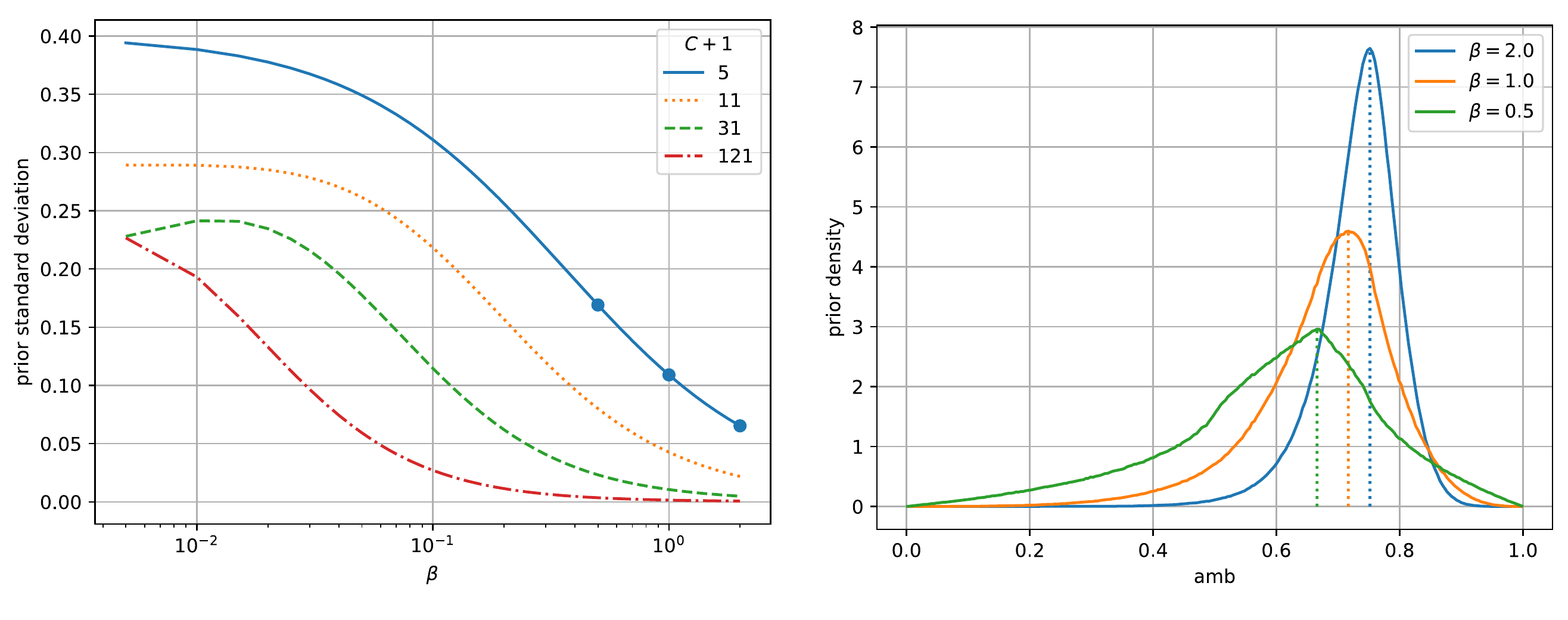}
    \caption{
    Influence of the Dirichlet hyperparameter $\beta$ on the prior distribution of ambiguity. 
    The left panel depicts the standard deviation of the ambiguity measure as a function of $\beta$, shown for varying numbers of answer categories $C+1$. Lower values of $\beta$ correspond to increased prior uncertainty, particularly for smaller category counts.  
    The right panel illustrates the induced prior densities over ambiguity for selected values of $\beta$ at a fixed category count of $C+1=5$. Smaller $\beta$ values result in broader and flatter prior distributions, whereas larger values produce more concentrated priors. Vertical lines indicate the respective modes.
    }
    \label{fig:priors}
\end{figure}

Ambiguity, as defined previously, captures the aleatoric uncertainty inherent in the labeling process. This refers to variability that originates from the population itself. In the following, we turn our attention to epistemic uncertainty, which reflects uncertainty about the ambiguity measure due to limited sample size.

Within the Bayesian framework established earlier, this form of uncertainty is represented by the posterior distribution of ambiguity. We now examine how the choice of prior, and in particular the hyperparameter $\beta$ of the Dirichlet distribution, influences this distribution. Our goal is to understand how different values of $\beta$ affect the concentration, spread, and credibility of the resulting ambiguity estimates.

In the absence of task-specific prior knowledge, we typically assume an \textit{uninformative prior} of the form $\mathrm{Dir}(\beta \pmb{1}_{C+1})$, where $\beta$ is a hyperparameter and $\pmb{1}_{C+1}$ denotes the $(C+1)$-dimensional vector of ones (one for each category, including a potential residual class such as \texttt{cant\_solve}). In the language of statistical physics, $\beta$ is often interpreted as a dimensionless inverse temperature.

A natural question arises: how should the hyperparameter $\beta$ be chosen? As shown earlier, both the posterior mean and mode are, in general, biased estimators of ambiguity, and the choice of $\beta$ affects the extent of this bias. More generally, $\beta$ influences the location, scale, and shape of the prior distribution over ambiguity---this distribution is typically far from uniform. As a result, the prior can significantly impact inference, especially when the sample size is small. Related difficulties in Bayesian inference of entropy have been noted by Nemenmann et al.\ \cite{nemenman2001entropy}.

The left panel of Figure~\ref{fig:priors} shows how the choice of $\beta$ affects the scale (i.e., standard deviation) of the prior distribution over ambiguity, for various numbers of answer categories. Three points on the blue curve highlight the standard deviation of the prior distribution for $C+1 = 5$ categories and $\beta \in \{0.5, 1.0, 2.0\}$. The corresponding prior distributions (obtained via sampling) are shown in the right panel. As $\beta$ increases, the prior shifts toward higher ambiguity values, expressing a belief that high ambiguity is more plausible. At the same time, the distribution becomes narrower.

These plots suggest that smaller values of $\beta$ may lead to more plausible prior beliefs about ambiguity. Nevertheless, we emphasize that there is no universally optimal choice of prior. This is also reflected in the bias behavior of the resulting estimators, as illustrated in Figure~\ref{fig:biases}. In this paper, we adopt the convention of using the prior $\mathrm{Dir}(\pmb{1}_{C+1})$, i.e., $\beta = 1$, which corresponds to a uniform prior over the simplex of $\vb{q}$ and appears to be a reasonable compromise. However, we explicitly acknowledge that improved priors or more elaborate inference techniques are promising directions for future work, but lie beyond the scope of this paper.

\subsection{Examples of the Posterior Distribution}

We illustrate the shape of the posterior distribution of the transformed random variable $\mathcal{a}(\vb{q})$ under a uniform Dirichlet prior, where $\mathcal{a}$ denotes one of the proposed ambiguity measures. We focus here on the binary case, i.e., with $C + 1 = 3$ categories. In this setting, the expression for the posterior density of both the standard and modified ambiguity measures can be written as a univariate integral involving Beta densities. The derivation is provided in Lemma~\ref{lemma:analytic_density} in Appendix~\ref{closed_form_binary_posterior}. The resulting posterior form under a uniform prior is reproduced here. 

Let $\vb{n} = (n^+, n^-, n^{\mathrm{cs}})$ denote the observed label counts in the annotation task. Then, for $0 < a < 1$, the posterior density can be written as
\begin{IEEEeqnarray}{rCl}
    p_{\mathcal{a}(\vb{q})}\left(a \mid n^+, n^-, n^{\mathrm{cs}} \right)
    &=& \int_{g(a)}^a \mathrm{d}u \;
        f_{n^{\mathrm{cs}}+\beta, \, n^+ + n^- + 2\beta}(u) \, \partial_a \xi \, (a, u) \nonumber\\
    && \times \left[
        f_{n^+ + \beta, \, n^- + \beta}(1 - \xi(a,u))
        + f_{n^+ + \beta, \, n^- + \beta}(\xi(a,u))
    \right]
    \text{.}
    \label{eq:amb_posterior_analytic}
\end{IEEEeqnarray}
where $\beta = 1$ corresponds to the uniform prior. The lower integration limit is $g(a) = 0$ for the modified ambiguity measure $\mathcal{a} = \widetilde{\mathrm{amb}}$ and $g(a) = 2a - 1$ for $\mathcal{a} = \mathrm{amb}$. The function $\xi$ depends on both $a$ and the integration variable $u$ and is given by
\begin{align}
    \xi(a, u) = \frac{1}{2} - \frac{1}{2}
    \begin{cases}
        \sqrt{2\frac{1-a}{1-u}-1} & \text{for } \mathcal{a} = \mathrm{amb}, \\
        \sqrt{\frac{1-a}{1-u}} & \text{for } \mathcal{a} = \widetilde{\mathrm{amb}}.
    \end{cases}
\end{align}

Figure~\ref{fig:all_posterior_comparison} shows example posterior distributions for several observed count vectors $\vb{n}$. These posterior distributions are obtained directly from the analytic expression in~\eqref{eq:amb_posterior_analytic} using numerical integration (here: adaptive Simpson’s rule \cite{kuncir1962algorithm}). 
In addition to the posterior density, we also display the exact posterior mean and standard deviation, based on the closed-form expressions derived in Equations~\eqref{eq:expr_amb_final} and~\eqref{eq:expr_amb_var_final} for the standard ambiguity, and in Equation~\eqref{eq:expr_amb_mod_var} for the modified ambiguity.

\begin{figure}
	\centering
	\includegraphics[width=0.98\linewidth]{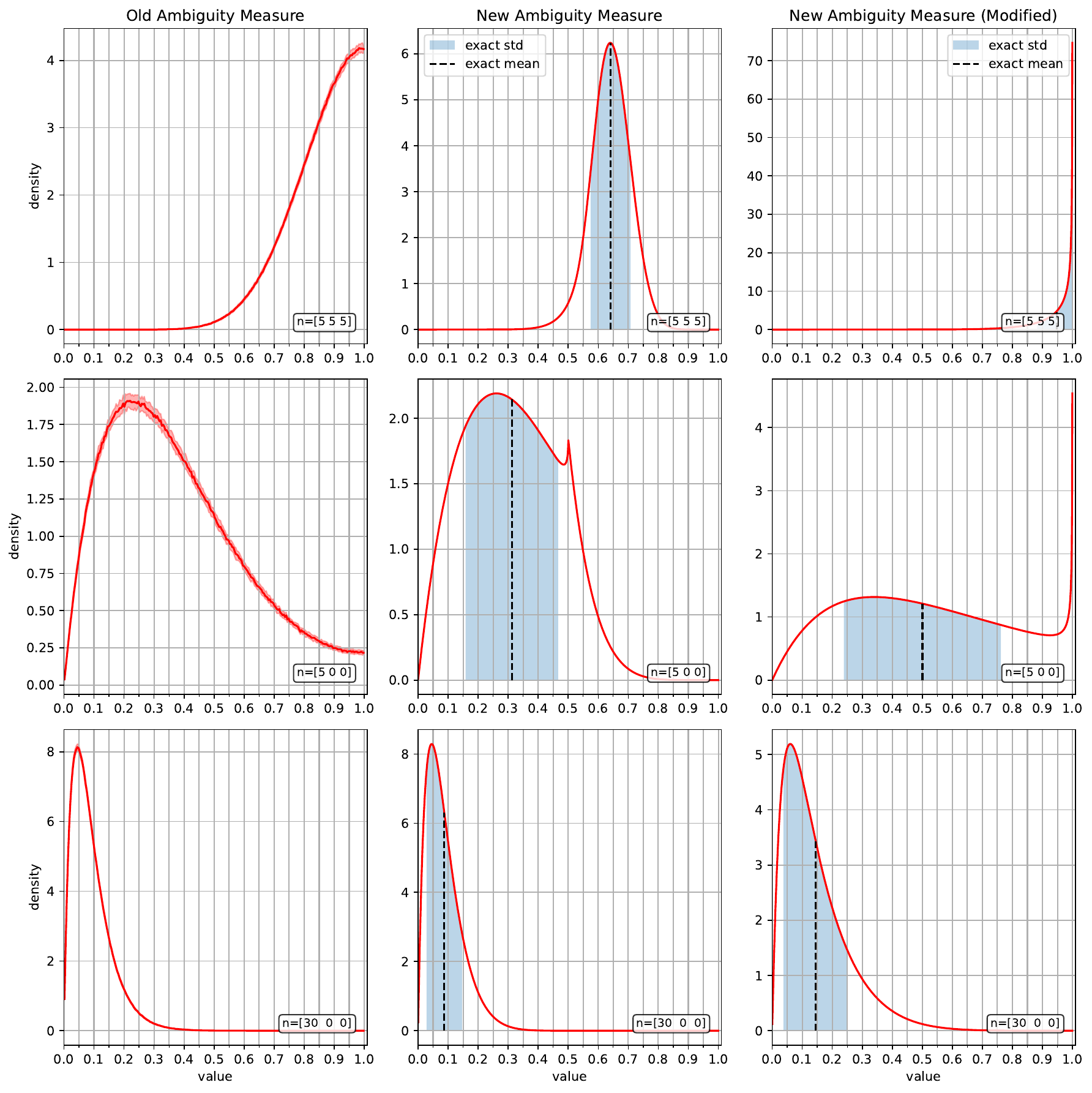}
	\caption{Posterior densities of the different ambiguity measures studied in this work, computed from observed answer frequencies in a binary annotation task (two outcome categories plus a \englanfz{can't solve} category.) Each row shows the same observations across the three measures: left---literature baseline; center---proposed measure; right---modified measure. For the center and right panels, the black dashed line and blue hatched region indicate the analytically derived posterior mean and $\pm1$ standard deviation.}
	\label{fig:all_posterior_comparison}
\end{figure}

For comparison, we also include approximate posterior densities of $\mathrm{amb}_0(\vb{q})$, the previously discussed ambiguity measure that is more common in the literature. To estimate its distribution, we draw $100{,}000$ samples from the analytically known Dirichlet posterior of the probability vectors $\vb{q}$, and estimate the density using $256$ equally spaced histogram bins.
To quantify Monte Carlo uncertainty, we repeat this sampling procedure $100$ times, then display the median (solid red line) and the interquartile range (red shaded region).

All three ambiguity measures exhibit qualitatively and quantitatively distinct behavior, especially for small sample sizes (e.g., the middle row in Figure~\ref{fig:all_posterior_comparison}, corresponding to only five repetitions) or diffuse response signals, i.e., nearly uniform distributions across proper categories with a substantial proportion of \englanfz{can't solve} responses (top row). As the number of repetitions increases, the posterior distributions converge, particularly when observations are consistent with low ambiguity. Notably, the proposed measure (center column) has a posterior density that vanishes for large values, a general property under our chosen prior (see Lemma~\ref{lemma:asymptotics_of_posterior}, Appendix~\ref{closed_form_binary_posterior}), whereas the modified measure (right column) exhibits divergence near its upper bound, which becomes negligible for large sample sizes.

A further distinctive feature of the proposed measure is a non-differentiable point at $0.5$, where the lower limit of the integral in Equation~\eqref{eq:amb_posterior_analytic} changes continuously but with a kink. More generally,  Figure~\ref{fig:all_posterior_comparison} illustrates that posterior mean and standard deviation capture only part of the distributions' characteristics: the posteriors may be skewed or have long tails, so a Normal approximation is not always reliable, even for two- or three-category tasks. Nevertheless, these summary statistics provide a rough guide to the posterior's shape without requiring full computation, particularly for the proposed measure.

\section{Related Work}

The study of ambiguity, entropy, and uncertainty quantification in the context of soft labels has gained significant traction in recent years. Soft labels, which represent distributions over possible answers rather than definitive classifications, allow for a nuanced understanding of annotator uncertainty and the inherent ambiguities present in annotation tasks. One of the foundational frameworks for understanding the implications of disagreement among annotators is the \textsc{CrowdTruth} methodology \cite{inel2014crowdtruth}. This approach posits that disagreement among human annotators can be indicative of the quality of the data being collected, as it acknowledges that human interpretation is often subjective and context-dependent \cite{aroyo2015truth, inel2014crowdtruth}.
%Subsequent work, such as Dumitrache et al.\ \cite{dumitrache2018}, extended this methodology by proposing quality metrics for crowdsourcing tasks and demonstrating its applicability to domains such as medical relation extraction \cite{dumitrache2018crowdsourcing}.
Subsequent work, such as Dumitrache et al.\ \cite{dumitrache2018}, expanded this methodology by introducing the CrowdTruth metrics, which capture the inter-dependencies between workers, input data, and annotations to quantify ambiguity in crowdsourcing tasks. These metrics have been applied to various domains, including medical relation extraction \cite{dumitrache2018crowdsourcing}.
In exploring the design of annotation tasks, Pradhan et al.\ \cite{pradhan2022search} highlight the need to clarify guidelines to mitigate ambiguity in responses from crowd workers. Their research demonstrates that even seemingly straightforward tasks can harbor subtle nuances that lead to varied interpretations among annotators. This aligns with the broader understanding that ambiguity is not merely a nuisance, but can be systematically studied and modeled to improve the quality of annotated datasets.

The concept of entropy plays a critical role in quantifying uncertainty, particularly in the context of discrete probability distributions, and is closely related to the notion of label ambiguity. While traditional Shannon entropy \cite{shannon1948mathematical} provides a fundamental understanding of information content, generalized entropies, such as the class of power entropies including quadratic entropy, have been proposed and extensively studied \cite{vajda2007generalized}. These generalized entropies have found applications in various domains, including biology and bioinformatics, where they are used as diversity measures \cite{zvarova2006genetic, grabchak2017generalized, derian2022tsallis}.
Quadratic entropy, a key component of our proposed ambiguity measure, belongs to this family of power entropies.  The construction of sample estimators for these diversity measures has garnered significant interest.
The Miller-Madow estimator \cite{miller1955note} is a classical and widely used estimator that analytically corrects the bias in naive (plug-in) entropy estimates. Carlton \cite{carlton1969bias} suggests a bias correction by expanding the logarithm, which provides a more generally valid but less accessible approximation. Horáček \cite{horavcek2011traditional} discusses an estimator for diversity measures and power entropies in the context of bioinformatics.
Nemenmann et al.\ \cite{nemenman2001entropy} study the properties of (nearly) uniform Dirichlet priors for entropy inference in a Bayesian framework and demonstrate numerous problems associated with this approach. Their findings are particularly relevant to our research on finite-sample estimation of ambiguity in soft labels, as we encounter fundamentally similar challenges in our analysis.

\section{Discussion}

Our central contribution is an interpretable ambiguity measure derived from annotator abstention (\englanfz{can't solve}) and explicit disagreement. We now discuss the practical utility of this measure, its methodological limitations compared to richer models, and its broader relevance in modern machine learning contexts.

\paragraph{Practical utility.} Because ambiguity values are probability-valued, they are straightforward to interpret, threshold, and communicate. This makes the measure actionable for workflows such as data curation (e.g., filtering items with $\widehat{\mathrm{amb}}>0.4$), benchmark stratification (constructing test slices by ambiguity), and active learning (prioritizing ambiguous items for adjudication). Ranking by ambiguity also helps surface dataset pain points and target annotation effort more efficiently. Formally, the ambiguity measure is a population-level functional of the underlying label distribution, but in practice we only observe finite samples. Our Bayesian Dirichlet--multinomial framework therefore provides both point estimates and posterior distributions, allowing practitioners to separate aleatoric ambiguity from epistemic uncertainty: with sparse annotations, posterior summaries remain wide and highlight the need for more data; with larger samples, estimates concentrate quickly.

\paragraph{Limitations and relationship to richer models.} Incorporating an explicit \englanfz{can't solve} response requires annotation protocols that permit abstention, which may necessitate modest UI changes and calibration; strategic over-use can be mitigated with standard quality controls (qualification tests, attention checks, audits). Methodologically, our Dirichlet--multinomial baseline is intentionally simple and does not model rater-specific heterogeneity or hierarchical task structure. Richer approaches (e.g., hierarchical Bayesian, latent-rater, or annotator-embedding models) can capture these effects and may improve downstream performance in some settings, but they come at the cost of additional complexity and reduced accessibility. We therefore position our measure as a practical, interpretable baseline that complements---rather than supplants---more elaborate modeling when those are warranted.

\paragraph{Broader context.} In modern ML contexts (notably large language models), the capacity to admit \emph{not knowing} is increasingly important, whether for uncertainty quantification, hallucination detection, or explicit abstention mechanisms \cite{xiong2024llmsexpressuncertaintyempirical,farquhar2024detecting,yadkori2024mitigatingllmhallucinationsconformal}. Conceptually, the ambiguity measure is defined on the true (soft-label) population distribution, but in practice it must be inferred from empirical response data. Our framework thus attempts to capture genuine ambiguity as reflected in annotator behavior while also quantifying the sampling variability and noise inherent in finite annotation sets. This dual perspective provides a pragmatic bridge between dataset practices and downstream model needs, yielding actionable signals for dataset maintenance and model training.

\section{Summary}

We introduced a novel, interpretable scalar measure, $\mathrm{amb}(\mathbf{q})$, to quantify aleatoric uncertainty in categorical soft labels. Derived from the probability of a label flip or an explicit \englanfz{can't solve} abstention, this measure asymmetrically extends quadratic entropy to separate uncertainty arising from class-level indistinguishability versus task unresolvability. We rigorously defined this measure, its normalized variant $\widetilde{\mathrm{amb}}(\mathbf{q})$, and presented a comprehensive comparative analysis against a literature baseline. Crucially, we developed a statistical inference framework using the Bayesian Dirichlet--multinomial model, providing closed-form analytical expressions for the expected value and variance of $\mathrm{amb}(\mathbf{q})$. This framework allows practitioners to infer the full posterior distribution, thereby disentangling aleatoric uncertainty (task-inherent) from epistemic uncertainty (sampling error). Positioned as a practical, probability-valued, and actionable baseline, $\mathrm{amb}(\mathbf{q})$ offers immediate utility for data curation, benchmark stratification, and active learning workflows.

\printbibliography

\newpage
\appendix
\section{Properties of Beta and Dirichlet Distributions}

\subsection{Moments of the Beta Distribution}
\label{beta_dist}

We now consider a few elementary properties of the Beta distribution. Let $\mathrm{B}(\alpha, \beta)$ denote the Beta function, defined as
\begin{align}
\mathrm{B}(\alpha, \beta) = \frac{\Gamma(\alpha)\,\Gamma(\beta)}{\Gamma(\alpha+\beta)},
\end{align}
where $\Gamma$ represents the Gamma function. The probability density function of the Beta distribution is given by
\begin{align}
f_{\alpha, \beta}(\pi) = \frac{1}{\mathrm{B}(\alpha, \beta)} \pi^{\alpha - 1} (1-\pi)^{\beta-1},
\end{align}
for $\pi \in [0, 1]$. Let $\pi \sim \mathrm{Beta}(\alpha, \beta)$, meaning $\pi$ is a random variable following a Beta distribution with parameters $\alpha$ and $\beta$. 
Due to the definition of the Beta function, it becomes straightforward to determine the moments of the distribution. For any $n \in \mathbb{N}_{\ge 0}$, we have
\begin{align}
E\left(\pi^n\right) = \frac{1}{\mathrm{B}(\alpha, \beta)} \int_0^1 \pi^n \, \pi^{\alpha-1} \, (1-\pi)^{\beta - 1} \, \dd \pi
= \frac{\mathrm{B}(\alpha+n, \beta)}{\mathrm{B}(\alpha, \beta)}.
\end{align}
Using the definition of the Beta function and exploiting the property of the Gamma function, $\Gamma(z+1) = z \Gamma(z)$, we can derive
\begin{align}
E\left(\pi^n\right) = \frac{ \prod_{i=0}^{n-1}(\alpha+i)  }{\prod_{i=0}^{n-1} (\alpha + \beta + i)}\text{.}\label{eq:moments_of_beta_dist}
\end{align}
The first and second moments of the distribution are immediately obtained as
\begin{align}
E(\pi) = \frac{\alpha}{\alpha+\beta}, \quad E\left(\pi^2\right) = \frac{\alpha \, (\alpha+1)}{(\alpha +\beta)(\alpha+\beta+1)}
\label{eq:first_and_second_moment}
\end{align}
From this, we also derive
\begin{align}
\mathrm{Var}(\pi) = E\left(\pi^2\right) - \left( E\left( \pi \right) \right)^2 = \frac{\alpha \, \beta}{(\alpha + \beta)^2 (\alpha + \beta + 1)}.
\label{eq:var_beta}
\end{align}
In the main text, we encounter random variables of the form $1 - \pi$. For $\pi \sim \mathrm{Beta}(\alpha, \beta)$, it holds that $1 - \pi \sim \mathrm{Beta}(\beta, \alpha)$, and the moments are obtained by swapping the roles of $\alpha$ and $\beta$. Furthermore, we need the expectation of $\pi(1-\pi)$. This can be easily determined as
\begin{align}
E\left( \pi \, (1-\pi) \right) = E\left( \pi \right) - E\left( \pi^2 \right) = \frac{\alpha \, \beta}{(\alpha+\beta)(\alpha+\beta+1)}.
\label{eq:mixed_prod_expectation}
\end{align}

\subsection{Special Moments of the Dirichlet Distribution}
\label{dirichlet_dist}

The Dirichlet distribution is the multidimensional generalization of the Beta distribution. Accordingly, many similarities arise in the calculation of expectations. Instead of the Beta function $\mathrm{B}(\alpha, \beta)$, we now encounter the generalized Beta function, which is defined for a parameter vector $\pmb{\alpha} = (\alpha_1, \dots, \alpha_C)^{\intercal}$ as
\begin{align}
\mathrm{B}(\pmb{\alpha}) = \frac{ \prod_{k=1}^C \Gamma(\alpha_k) }{ \Gamma\left( \sum_{k=1}^C \alpha_k \right)}.
\end{align}
The density of the Dirichlet distribution is proportional to $\prod_{k=1}^C p_k^{\alpha_k - 1}$, where $\vb{p} = (p_1, \dots, p_C)^{\intercal}$ is an element of the standard simplex $\Delta_{C-1}$, meaning all entries $p_k$ lie in $[0, 1]$ and sum to $1$. The normalization constant is given by the generalized Beta function.

In the main text, we encounter the need to compute expectations of the form $E\left( p_k^t p_l^s\right)$, where $k, l$ denote some indices and $s, t$ represent arbitrary powers. Such expectations can be easily determined, analogous to the moment calculations for the Beta distribution, by exploiting the properties of the Gamma distribution. First, from the definition of the (generalized) Beta function, we have
\begin{equation}
\begin{split}
E\left( p_k^s p_l^t \right) = &\frac{ \Gamma\left( \sum_{m=1}^C \alpha_m \right)}{ \Gamma\left((s+t)+\sum_{m=1}^C \alpha_m\right) \prod_{m=1}^C \Gamma(\alpha_m) } \\
&\times \begin{cases}
    \Gamma((s+t)+\alpha_k)\prod_{m\neq k} \Gamma(\alpha_m) &\text{if}\, k=l,\\
    \Gamma(s+\alpha_k)\Gamma(t+\alpha_l) \prod_{\substack{m\neq l \\ m\neq k}} \Gamma(\alpha_m) & \text{else.}
\end{cases}
\end{split}
\end{equation}
Then, using the property $\Gamma(z+1)=z\Gamma(z)$, we can eliminate the Gamma functions and obtain the expected value as
\begin{equation}
\begin{split}
    E\left( p_k^s p_l^t \right) = \frac{1}{\prod_{i=0}^{s+t-1} (\alpha_0 + i) }
    \begin{cases}
        \prod_{i=0}^{s+t-1} (\alpha_k + i) & \text{if}\, k=l\text{,} \\
        \left( \prod_{i=0}^{s-1} (\alpha_k + i) \right)\left( \prod_{j=0}^{t-1} (\alpha_l + j) \right) & \text{else,}
    \end{cases}
    \label{eq:prod_moments_dirichlet}
\end{split}
\end{equation}
where $\alpha_0 \equiv \sum_{k=1}^C \alpha_k$.

\section{Proofs}

\subsection{Strictly Monotonic Consistency of the Plug-In Estimator for $\mathrm{amb}$}
\label{plugin_consistency}

In this appendix, we show that the bias of the frequentist plug-in estimator for ambiguity converges strictly monotonically to zero from below. This is formalized in the following lemma.

\begin{lemma}
Let $\widehat{\mathrm{amb}}_n = \mathrm{amb}(\hat{\vb{q}}_n)$ denote the plug-in estimator of ambiguity based on a sample of size $n$, where $\mathrm{amb}(\cdot)$ is the ambiguity functional defined in Equation~\eqref{eq:def_amb_final}, and $\hat{\vb{q}}_n$ is the empirical distribution of the observed labels. Define the estimator's bias as
\[
\mathrm{bias}_n =E_{\vb{q}}\left(\widehat{\mathrm{amb}}_n\right) - \mathrm{amb}(\vb{q}),
\]
where $\vb{q}$ is the true underlying probability vector. Then, for all non-degenerate $\vb{q}$, the sequence $(\mathrm{bias}_n)_{n \in \mathbb{N}}$ converges strictly monotonically from below to $0$ as $n \to \infty$.
\end{lemma}

\begin{proof}
Let $\vb{q}$ be a non-degenerate probability vector, meaning that $\vb{q}$ is not a one-hot vector. Note that if $\vb{q}$ is one-hot---either on one of the proper categories or on the $\texttt{cant\_solve}$ category---the plug-in estimator is unbiased for all $n$. We therefore exclude these trivial cases from the following analysis.
Then the true ambiguity is given by the expression in Equation~\eqref{eq:conditional_amb}. The expected value of the plug-in estimator for this case was derived in Equation~\eqref{eq:plugin_estimator_mean}. By straightforward algebraic manipulation, the bias of the estimator can be written as
\begin{align}
    \mathrm{bias}_n = - \frac{1 - q_{\mathrm{cs}}^n}{n} \left\{ 1 - \frac{\sum_{k=1}^C q_k^2}{(1 - q_{\mathrm{cs}})^2} \right\}.
\end{align}
We now analyze the sign and convergence behavior of this expression.

First, the factor $f_n \equiv n^{-1} \, (1 - q_{\mathrm{cs}}^n)$ is strictly positive for all $n$ and $q_{\mathrm{cs}} < 1$. The second factor is strictly positive according to the binomial theorem, the fact that $(1-q_{\mathrm{cs}})^2 = \left(\sum_{k=1}^C q_k\right)^2$ and our assumption of non-degenerate $\vb{q}$.
Thus, the overall bias is strictly negative for all $n$, which shows that the plug-in estimator systematically underestimates the true ambiguity.

To establish strict monotonic convergence of $\mathrm{bias}_n$ to zero from below, it suffices to show that the sequence $(f_n)$ is strictly decreasing. To that end, observe that
\begin{align}
    f_n = \frac{1 - q_{\mathrm{cs}}^n}{n} = \frac{1 - q_{\mathrm{cs}}}{n} \sum_{k=0}^{n-1} q_{\mathrm{cs}}^k.
\end{align}
If $q_{\mathrm{cs}}$ is identically zero, then $f_n$ is clearly strictly monotonically decreasing. Therefore, we can concentrate on the case where $q_{\mathrm{cs}}$ is strictly greater than zero.
Now, define $a_k \equiv q_{\mathrm{cs}}^k$. Since $0 < q_{\mathrm{cs}} < 1$, the sequence $(a_k)$ is strictly decreasing. It follows that
\begin{align}
    \frac{1}{n} \sum_{k=0}^{n-1} a_k > \frac{1}{n+1} \sum_{k=0}^{n} a_k,
\end{align}
because the additional term $a_n$ is strictly smaller than the average of the first $n$ terms. Hence, the sequence $(f_n)$ is strictly decreasing, which implies that $(\mathrm{bias}_n)$ increases strictly monotonically toward zero.
This completes the proof.
\end{proof}

\subsection{Distribution of Ambiguity under Dirichlet-Distributed Probability Vectors in the Special Case $C+1 = 3$}

\label{closed_form_binary_posterior}

Deriving the exact form of the posterior distribution of ambiguity is challenging in the general categorical case, as it involves evaluating high-dimensional integrals over subsets of the standard simplex. However, for the special case involving two proper answer categories plus the residual \texttt{cant\_solve} category---that is, for $C+1 = 3$---the distribution of ambiguity under Dirichlet-distributed probability vectors can be stated more explicitly. This is the subject of the following lemma.

\begin{lemma}
Let $\vb{q} \sim \mathrm{Dir}(\alpha, \beta, \alpha_{\mathrm{cs}})$ be a Dirichlet-distributed random vector over three categories (two proper categories and one \texttt{cant\_solve} category). Let $\mathcal{a}(\vb{q})$ denote either the standard ambiguity measure defined in Equation~\eqref{eq:def_amb_final} or the modified ambiguity defined in Equation~\eqref{eq:def_modified_amb}. Then, the density of the transformed random variable $\mathcal{a}(\vb{q})$ is given by:
\begin{align*}
    p_{\mathcal{a}(\vb{q})}\left( a \mid \alpha, \beta, \alpha_{\mathrm{cs}} \right)
    = \int_{g(a)}^a \dd u \; f_{\alpha_{\mathrm{cs}}, \alpha+\beta}(u) \left[ f_{\alpha,\beta}(1-\xi) + f_{\alpha,\beta}(\xi) \right] \pdv{\xi}{a},
\end{align*}
where $f_{\alpha, \beta}$ denotes the probability density function of the Beta distribution $\mathrm{Beta}(\alpha, \beta)$. The lower bound of integration $g(a)$ depends on the ambiguity variant:
\begin{itemize}
    \item $g(a) = 0$ for the modified ambiguity $\widetilde{\mathrm{amb}}$,
    %\item $g(a) = \max\{0, 2a - 1\}$ for the standard ambiguity $\mathrm{amb}$.
    \item $g(a) = 2a-1$ for the standard ambiguity $\mathrm{amb}$.
\end{itemize}
The function $\xi = \xi(a, u)$ and its partial derivative with respect to $a$, $\partial_a \xi(a, u)$, are given by
\begin{align*}
    \xi(a,u) &= \frac{1}{2} \left( 1 - \sqrt{2 \cdot \frac{1-a}{1-u} - 1} \right), \\
    \partial_a \xi(a,u) &= \frac{1}{2\sqrt{(1-u)[(1-a) + (u-a)]}}.
\end{align*}
for the standard ambiguity $\mathcal{a} = \mathrm{amb}$, and
\begin{align*}
    \xi(a,u) &= \frac{1}{2} \left( 1 - \sqrt{\frac{1-a}{1-u}} \right), \\
    \partial_a \xi(a,u) &= \frac{1}{4\sqrt{(1-a)(1-u)}}.
\end{align*}
for the modified ambiguity $\mathcal{a} = \widetilde{\mathrm{amb}}$.
\label{lemma:analytic_density}
\end{lemma}

\begin{proof}
To derive the density of the random variable transformed by one of the ambiguity measures, we begin with the cumulative distribution function. In full generality, for arbitrary $C$ and a Dirichlet-distributed probability vector $\vb{q} \sim \mathrm{Dir}(\pmb{\alpha})$, the CDF can be written as
\begin{align}
    P\left(\mathcal{a}(\vb{q}) \le a \mid \pmb{\alpha} \right) = \int_{\mathcal{A}} \dd^C q \, f_{\pmb{\alpha}}(\vb{q}),
    \label{eq: cdf_integral}
\end{align}
where $f_{\pmb{\alpha}}$ is the density of the Dirichlet distribution $\mathrm{Dir}(\pmb{\alpha})$, and $\mathcal{A} \subset \Delta_C$ is the subset of the standard $C$-simplex for which $\mathcal{a}(\vb{q}) \le a$.

We apply the natural coordinate transformation
\begin{align*}
    &\vb{q} = (q_1, \dots, q_C, q_{\mathrm{cs}})^{\intercal} \mapsto (u, \vb{p}^{\intercal})^{\intercal}, \\
    &\text{where} \quad u = q_{\mathrm{cs}}, \quad \text{and} \quad p_k = \frac{q_k}{1 - q_{\mathrm{cs}}}, \quad \text{for } k = 1, \dots, C.
\end{align*}
Under this transformation, the integral in \eqref{eq: cdf_integral} becomes
\begin{align}
    P\left(\mathcal{a}(\vb{q}) \le a \mid \pmb{\alpha} \right)
    = \int_0^1 \dd u \, f_{\alpha_{\mathrm{cs}}, \alpha_0 - \alpha_{\mathrm{cs}}}(u) \int_{\mathcal{B}} \dd^{C-1} p \, f_{\alpha_1, \dots, \alpha_C}(\vb{p}),
\end{align}
where $\mathcal{B} = \mathcal{B}(a, u; C)$ is the subset of the $(C{-}1)$-simplex given by
\begin{align}
    \mathcal{B} = \left\{ \vb{p} \in \Delta_{C-1} \; : \; \mathcal{a}((1-u)p_1, \dots, (1-u)p_C, u) \le a \right\}.
    \label{eq:generic_integral}
\end{align}

We now specialize to the case $C = 2$, and focus on the standard ambiguity $\mathcal{a} = \mathrm{amb}$. The argument for the modified ambiguity $\widetilde{\mathrm{amb}}$ proceeds analogously.

In this case, the inner integral becomes a univariate Beta integral. For $u < 1$, the set $\mathcal{B}$ consists of all $\pi \in [0, 1]$ such that
\begin{align}
    1 - (1-u)\left( \pi^2 + (1-\pi)^2 \right) \le a
    \quad \Leftrightarrow \quad
    \left(\pi - \frac{1}{2}\right)^2 \ge \frac{1}{4}\left(2\frac{1-a}{1-u} - 1\right).
    \label{eq:raw_integral_condition}
\end{align}
This condition is automatically satisfied if the right-hand side of the inequality is nonpositive, which occurs exactly when $u \le 2a - 1$. Conversely, $\mathcal{B} = \emptyset$ if $u > a$. Since $a \le 1$, it follows that $2a - 1 \le a$, determining the bounds of integration. We thus obtain:
\begin{align}
    P\left(\mathrm{amb}(\vb{q}) \le a \mid \alpha, \beta, \alpha_{\mathrm{cs}} \right)
    &= \int_{-\infty}^{2a - 1} \dd u \, f_{\alpha_{\mathrm{cs}}, \alpha+\beta}(u) \nonumber\\
    &\quad + \int_{2a - 1}^a \dd u \, f_{\alpha_{\mathrm{cs}}, \alpha+\beta}(u)
        \int_{\mathcal{B}} \dd \pi \, f_{\alpha, \beta}(\pi).
    \label{eq:intermediate_integral}
\end{align}

For $2a - 1 < u < a$, the inequality in \eqref{eq:raw_integral_condition} defines a symmetric exclusion band around $\pi = 1/2$. Specifically,
\begin{align}
    \mathcal{B} = \left\{ \pi \in [0, 1] : \abs{\pi - \tfrac{1}{2}} \ge \tfrac{1}{2} \sqrt{2\frac{1-a}{1-u} - 1} \right\}.
\end{align}
Letting $\xi^{\pm}(a, u) = \frac{1}{2} \left( 1 \pm \sqrt{2 \frac{1-a}{1-u} - 1} \right)$, we can write:
\begin{align}
    \int_{\mathcal{B}} \dd \pi \, f_{\alpha,\beta}(\pi)
    = \int_0^{\xi^{-}} \dd \pi \, f_{\alpha,\beta}(\pi)
    + \int_{\xi^{+}}^1 \dd \pi \, f_{\alpha,\beta}(\pi)
    = 1 - \left[ I_{\alpha,\beta}(\xi^{+}) - I_{\alpha,\beta}(\xi^{-}) \right],
\end{align}
where $I_{\alpha,\beta}(x)$ denotes the cumulative distribution function of $\mathrm{Beta}(\alpha, \beta)$.

To obtain the density $p_{\mathrm{amb}}(a)$, we differentiate the right-hand side of \eqref{eq:intermediate_integral} with respect to $a$. Applying the Leibniz integral rule (justified by the regularity of the integrands), we obtain:
\begin{align}
    p_{\mathrm{amb}}(a)
    &= 2 f_{\alpha_{\mathrm{cs}}, \alpha+\beta}(2a - 1) \nonumber\\
    &\quad + \left. f_{\alpha_{\mathrm{cs}}, \alpha+\beta}(a) \left[ 1 - \left( I_{\alpha,\beta}(\xi^{+}) - I_{\alpha,\beta}(\xi^{-}) \right) \right] \right|_{u=a} \nonumber\\
    &\quad - 2 \left. f_{\alpha_{\mathrm{cs}}, \alpha+\beta}(2a-1) \left[ 1 - \left( I_{\alpha,\beta}(\xi^{+}) - I_{\alpha,\beta}(\xi^{-}) \right) \right] \right|_{u=2a-1} \nonumber\\
    &\quad + \int_{2a - 1}^a \dd u \, f_{\alpha_{\mathrm{cs}}, \alpha+\beta}(u)
        \left[ f_{\alpha,\beta}(\xi^{-}) \pdv{\xi^{-}}{a}
              - f_{\alpha,\beta}(\xi^{+}) \pdv{\xi^{+}}{a} \right].
    \label{eq:pdf_intermediate_clean}
\end{align}
It is easy to check that $\xi^{-}(a, a) = 0$, $\xi^{+}(a, a) = 1$, so the second term vanishes. Also, at $u = 2a - 1$ we have $\xi^{\pm}(a, u) = 1/2$, and hence the first and third terms cancel. Finally, observe that
\[
    \xi^{+} = 1 - \xi^{-}, \quad \partial_a \xi^{+} = -\partial_a{\xi^{-}}\text{.}
\]
Hence, with $\xi(a, u) \equiv \xi^{-}(a, u)$, we arrive at the final form:
\begin{align}
    p_{\mathrm{amb}}(a)
    = \int_{2a - 1}^{a} \dd u \; f_{\alpha_{\mathrm{cs}}, \alpha+\beta}(u)
      \left[ f_{\alpha,\beta}(\xi(a,u)) + f_{\alpha,\beta}(1 - \xi(a,u)) \right]
      \pdv{\xi}{a},
\end{align}
as claimed.
\end{proof}

%% DEBUG
\begin{lemma}[Asymptotics at $a\to1^{-}$]
Suppose $\alpha,\;\beta,\;\alpha_{\mathrm{cs}}\ge1$, and let
\[
p_{\mathrm{amb}}(a)
= p_{\mathrm{amb}(\mathbf{q})}(a \mid \alpha,\beta,\alpha_{\mathrm{cs}})
\]
be the density of the standard ambiguity transform, and
\[
p_{\widetilde{\mathrm{amb}}}(a)
= p_{\widetilde{\mathrm{amb}}(\mathbf{q})}(a \mid \alpha,\beta,\alpha_{\mathrm{cs}})
\]
be the density of the modified ambiguity transform, as given in lemma \ref{lemma:analytic_density}. Then
\begin{align*}
\lim_{a\to1^{-}} p_{\mathrm{amb}}(a) = 0,
\quad \lim_{a\to1^{-}} p_{\widetilde{\mathrm{amb}}}(a) = +\infty.
\end{align*}
\label{lemma:asymptotics_of_posterior}
\end{lemma}

\begin{proof}
We treat each case separately.\vspace{-\baselineskip}
\paragraph{(i) Standard ambiguity.}
Let $0 < \varepsilon \ll 1$
Recall from the previous lemma that
\[
p_{\mathrm{amb}}(a)
= \int_{g(a)}^{a}
    f_{\alpha_{\mathrm{cs}},\,\alpha+\beta}(u)\,
    \bigl[f_{\alpha,\beta}(1-\xi) + f_{\alpha,\beta}(\xi)\bigr]\,
    \frac{\partial \xi}{\partial a}(a,u)\,
    \mathrm{d}u,
\]
with $g(a)=2a-1$ and
\[
\frac{\partial \xi}{\partial a}(1-\varepsilon,u)
= \frac{1}{2\sqrt{(1-u)\bigl[2\varepsilon - (1-u)\bigr]}}.
\]
Since $\alpha,\beta,\alpha_{\mathrm{cs}}\ge1$, the Beta densities
$f_{\alpha_{\mathrm{cs}},\alpha+\beta}(\,\cdot\,)$ and $f_{\alpha,\beta}(\,\cdot\,)$
are bounded above. Thus, for some constant $K>0$,
\[
p_{\mathrm{amb}}(1-\varepsilon)
\le K
\int_{1-2\varepsilon}^{1-\varepsilon}
\frac{\mathrm{d}u}{\sqrt{2\varepsilon - (1-u)}}.
\]
Note that we have absorbed the additional factor $(1-u)^{-1/2}$ from $\partial_a \xi$ into the first beta density.
Perform a variable substitution to see that
\[
p_{\mathrm{amb}}(1-\varepsilon)
\le K \int_{0}^{\varepsilon} s^{-1/2}\,\mathrm{d}s
= 2K\sqrt{\varepsilon}
\;\xrightarrow{\;\varepsilon\to0\;}0.
\]
\paragraph{(ii) Modified ambiguity.}
Here, the lower limit of integration is given by $g(a) = 0$. The density of $\widetilde{\mathrm{amb}}$, evaluated at $0 < a < 1$, can be rewritten as
\begin{align*}
    p_{\widetilde{\mathrm{amb}}}(a) = a \int_0^1 \mathrm{d}\omega \; f_{\alpha_{\mathrm{cs}}, \alpha+\beta}(\omega a) 
    \left.\pdv{\xi}{a}\right|_{(a, \omega a)} 
    \left[ f_{\alpha, \beta}(\xi(a, \omega a)) + f_{\alpha, \beta}(1 - \xi(a, \omega a)) \right],
\end{align*}
where we have used the substitution $u \mapsto \omega = u/a$.

Let $0 < \varepsilon \ll 1$ and $\omega \in (0, 1)$. Define
\begin{align*}
    v(\varepsilon, \omega) \equiv \xi (1-\varepsilon, \omega(1-\varepsilon)) = \frac{1}{2}\left(1 - \sqrt{\frac{\varepsilon}{1 - \omega(1-\varepsilon)}}\right),
\end{align*}
which clearly converges to $1/2$ as $\varepsilon \to 0$, for any admissible $\omega$.

The partial derivative of $\xi$ appearing under the integral, evaluated at $a = 1 - \varepsilon$, is given by
\begin{align}
    \partial_a \xi \, (1 - \varepsilon, \omega(1-\varepsilon)) = \frac{1}{4} \varepsilon^{-1/2} (1 - \omega(1 - \varepsilon))^{-1/2}.
    \label{eq:partial_at_1_minus_eps}
\end{align}
We may absorb the final factor in \eqref{eq:partial_at_1_minus_eps} into the first Beta density under the integral. This yields
\begin{align}
    p_{\widetilde{\mathrm{amb}}}(1-\varepsilon) \propto (1 - \varepsilon) \, \varepsilon^{-1/2} \, I(\varepsilon),
    \label{eq:p_mod_amb_1_minus_eps}
\end{align}
where
\begin{align}
    I(\varepsilon) \equiv \int_0^1 f_{\alpha_{\mathrm{cs}}, \alpha+\beta - 1/2}(\omega(1 - \varepsilon)) \left[ f_{\alpha, \beta}(v(\varepsilon, \omega)) + f_{\alpha, \beta}(1 - v(\varepsilon, \omega)) \right] \, \mathrm{d}\omega.
    \label{eq:remaining_integral}
\end{align}
By assumption, all parameters $\alpha, \beta, \alpha_{\mathrm{cs}}$ are greater than or equal to $1$, so the Beta densities under the integral in \eqref{eq:remaining_integral} are bounded. In particular, this allows us to find an integrable majorant, so that we may interchange limit and integration via Lebesgue's dominated convergence theorem. Using the above asymptotics of $v(\varepsilon, \omega)$, together with the continuity of the Beta densities, we obtain
\begin{align}
    \lim_{\varepsilon \to 0} I(\varepsilon) = 2 f_{\alpha,\beta}(1/2) \underbrace{\int_0^1 f_{\alpha_{\mathrm{cs}}, \alpha+\beta - 1/2}(\omega) \, \mathrm{d}\omega}_{=1} = \mathrm{const.} > 0.
\end{align}
From \eqref{eq:p_mod_amb_1_minus_eps} we conclude that
\begin{align}
    p_{\widetilde{\mathrm{amb}}}(1-\varepsilon) \sim K \varepsilon^{-1/2}, \quad K > 0,
\end{align}
as $\varepsilon \to 0$. In particular, this shows that $p_{\widetilde{\mathrm{amb}}}(a) \to \infty$ as $a \to 1^-$, as claimed.
\end{proof}

\end{document}